\documentclass[final,12pt]{colt2025} 

\usepackage{hyperref}
\usepackage{url}
\usepackage{tcolorbox}
\usepackage{SJ_Definitions}
\usepackage{footmisc}
\title[Experimental Design for Semiparametric Bandits]{Experimental Design for Semiparametric Bandits}
\usepackage{times}
\newcommand{\algoa}{\texttt{DEO}}
\newcommand{\algob}{\texttt{SBE}}
\newcommand{\op}{\operatorname{op}}
\newcommand{\opt}{\operatorname{opt}}
\newcommand{\bSigma}{\bm{\Sigma}}

\newcommand{\des}{\operatorname{deo}}

\usepackage{etoc}
\etocdepthtag.toc{mtchapter}
\etocsettagdepth{mtchapter}{subsection}
\etocsettagdepth{mtappendix}{none}
\usepackage{SJ_Definitions}

\coltauthor{%
\Name{Seok-Jin Kim} \Email{seok-jin.kim@columbia.edu}\\
\addr Columbia University
\AND
\Name{Gi-Soo Kim} \Email{gisookim@unist.ac.kr}\\
\addr UNIST
\AND
\Name{Min-hwan Oh} \Email{minoh@snu.ac.kr}\\
\addr Seoul National University
}

\begin{document}

\maketitle
\begin{abstract}
We study finite-armed semiparametric bandits, where each arm's reward combines a linear component with an unknown, potentially adversarial shift. This model strictly generalizes classical linear bandits and reflects complexities common in practice. We propose the first experimental-design approach that simultaneously offers a sharp regret bound, a PAC bound, and a best-arm identification guarantee. Our method attains the minimax regret $\tilde{\mathcal{O}}(\sqrt{dT})$, matching the known lower bound for finite-armed linear bandits, and further achieves logarithmic regret under a positive suboptimality gap condition. These guarantees follow from our refined non-asymptotic analysis of orthogonalized regression that attains the optimal $\sqrt{d}$ rate, paving the way for robust and efficient learning across a broad class of semiparametric bandit problems.
\end{abstract}


\begingroup 
\renewcommand{\thefootnote}{}
\footnotetext{Corresponding authors: Gi-Soo Kim (gisookim@unist.ac.kr) and  Min-hwan Oh (minoh@snu.ac.kr).}%
\endgroup

\section{Introduction}
\noindent
Linear bandits, in which each arm is identified with a feature vector
\(x\in\mathbb{R}^d\) and the expected reward takes the form
\(x^\top \theta^\star\), have been extensively studied for both regret
minimization and pure-exploration objectives
\citep{abbasi2011improved,
soare2014best,jedra2020optimal,degenne2020gamification,xu2018fully,fiez2019sequential}.
Thanks to the linear structure, design-based methods such as the G-optimal design
provide strong theoretical guarantees, enabling
optimal \(\tilde{\mathcal{O}}(\sqrt{dT \log K})\)\footnote{\(\tilde{\mathcal{O}}\) suppresses logarithmic factors of $d$ and $T$ but \emph{does not hide dependence on the number of arms, $K$.}}
 regret bounds  \citep{LS19bandit-book} and tight sample
complexities for \textit{probably approximately correct} (PAC) learning \citep{li2022instance,chaudhuri2019pac,sakhi2023pac} and \textit{best arm
identification} (BAI) \citep{soare2014best,fiez2019sequential,jedra2020optimal,komiyama2022minimax}.

However, many practical systems face adversarial or unpredictable effects that
cannot be faithfully captured by a purely linear model. Baseline shifts
or adversarial interventions often introduce reward
variations beyond the reach of linear assumptions. To address these
complexities, a semiparametric reward model
\citep{greenewald2017action,krishnamurthy2018semiparametric,kim2019contextual}
incorporates an additional shift term. Specifically, if \(r_t\) denotes the reward obtained from taking action 
\(a_t\) whose feature vector is $x_{a_t} \in  \mathbb{R}^d$
at time \(t\), 
the model posits that
\[
r_t = x_{a_t}^\top \theta^\star + \nu_t + \eta_t,
\]
where \(\theta^\star\) is an unknown parameter, \(\nu_t\) is an arbitrary (potentially adversarial) shift that is not governed by a parametric form and is determined prior to action selection, and \(\eta_t\) is noise.

This formulation subsumes linear bandits as a
special case (by taking \(\nu_t=0\)) and captures extra complexities found in
real-world applications, such as recommender systems \citep{mladenov2020demonstrating}.
However, many real-world bandit problems involve fixed features with time-varying effects—for example, clinical trials with fixed treatments but varying subjects (but without subject information) \citep{kazerouni2021best}.
A similar scenario appears in ad selection for a landing page 
(without access to personal information).
The term $\nu_t$ accounts for such time-varying baselines, including changes in varying user propensity, external interference, or broader trends. Moreover, BAI is commonly studied under fixed features, where incorporating $\nu_t$ is still meaningful.

Existing analyses of semiparametric bandits currently provide only
\(\tilde{\mathcal{O}}(d\sqrt{T})\) bounds
\citep{greenewald2017action,krishnamurthy2018semiparametric,chowdhury2023gbose}
or \(\tilde{\mathcal{O}}(d^{3/2}\sqrt{T})\) bounds
\citep{kim2019contextual}. 
However, both rates fall short of the 
\(\tilde{\mathcal{O}}(\sqrt{dT \log K})\) bound known to be optimal in 
linear bandits with a finite action set~\citep{li2019nearly}. 
This gap leaves open the question of whether more sophisticated 
design-based methods can attain a sharper \(\sqrt{d}\)-type regret 
bound in semiparametric settings, even under adversarial shifts. 
Moreover, no gap-dependent logarithmic regret result has been established, 
further contrasting with the well-studied linear bandit framework.

Another key limitation in the existing literature is the complete absence of work 
on experimental design in semiparametric bandits.
Whereas G-optimal design is routinely employed in linear bandits to obtain 
PAC and BAI guarantees, no analogous design-based technique exists for semiparametric bandits. 
All prior approaches 
\citep{greenewald2017action,krishnamurthy2018semiparametric,kim2019contextual,chowdhury2023gbose} 
have utilized \emph{orthogonalized regression}, which introduces substantial 
new challenges to experimental design and prevents the direct extension 
of existing linear-bandit methods.
Hence, the following research questions arise:
\begin{itemize}
    \setlength\itemsep{0em}
    \item \textit{Can we develop an experimental design for semiparametric bandits?}
    \item \textit{Can we develop an efficient algorithm that leverages this experimental design 
          to simultaneously achieve sharp regret bounds, PAC, and BAI guarantees?
          }
\end{itemize}

In this paper, we address these questions by proposing new algorithmic 
and analysis frameworks for semiparametric bandits that incorporate optimal design methods 
into orthogonalized regression. 
We propose the first procedure that tackles
the non-convex design problem of orthogonalized regression efficiently, enabling a sharp regret
bound of \(\tilde{\mathcal{O}}(\sqrt{dT \log K})\). Moreover, by utilizing a dependency on the suboptimality gap, our approach also achieves logarithmic regret, bridging the gap between the theory of linear bandits and that of semiparametric bandits.
Beyond regret minimization, we provide the first PAC and BAI  results for semiparametric bandits by establishing sample complexities that match those of \citet{soare2014best} for linear bandits.
Finally, we propose our main algorithm, which exhibits low regret as well as PAC and BAI guarantees.
These new results are driven by a
sharper non-asymptotic analysis of orthogonalized regression, which replaces
the loose Cauchy--Schwarz arguments and delivers dimension-optimal
statistical rates.

Our result can also be applied to a multi-armed bandit (MAB) problem with a semiparametric reward model. 
In this setting, each arm $i$ has a base reward distribution with a mean of $\mu_i \subset \RR$ for $i=1, 2, \dots, K$, but the reward observed at each time step $t$ includes an additional, time-varying shift, $\nu_t$. This setup can be framed within our model by defining the feature set $\cX$ as the standard basis of $\RR^K$.
We discuss this application in greater detail in Section~\ref{section; application to MAB}. For this bandit problem, our results achieve a regret bound of $\tilde{O}(\sqrt{KT})$ as well as logarithmic regret when there is a suboptimality gap.

\subsection{Our Contributions}
\noindent
We summarize our key contributions as follows:
\begin{enumerate}[1.]
\item \textbf{Experimental design for semiparametric bandits.} \;
Prior to our work, experimental design for semiparametric reward models had not been investigated. 
We develop a novel design-based approach that, for the first time, enables both sharp regret bounds and exploration-based guarantees (PAC, BAI) in semiparametric bandits. In particular, our proposed algorithm, built around this new experimental design, attains a regret bound of  
\(\tilde \cO(\sqrt{dT\log K})\) while simultaneously ensuring PAC and BAI performance.

\item \textbf{\(\tilde{\cO}(\sqrt{dT \log K})\) regret bound for semiparametric bandits.} \;
We propose a phase-elimination-based algorithm that achieves 
\(\min(\tilde{\cO}(\sqrt{dT \log K}), \tilde{\cO}(d \sqrt{T}))\) cumulative regret 
(Theorem~\ref{theorem; regret bound without gap} and Theorem~\ref{theorem; adaptive regret bound}).
To our best knowledge, this is the first result attaining the \(\sqrt{d}\)-rate regret in a semiparametric bandit setting.
It matches the known lower bounds for linear bandits with finite arms \citep{LS19bandit-book}, 
and since the linear model is a special case of our framework, 
this bound is information-theoretically minimax optimal (up to logarithmic factors).

\item \textbf{Gap-dependent logarithmic regret bound for semiparametric bandits.} \,
By allowing for suboptimality gap dependence, we also derive the regret bound logarithmic in \(T\) (Theorem~\ref{theorem; regret bound with gap}), 
marking the first  logarithmic regret result for semiparametric bandits.
Hence, when the suboptimality gap is reasonably small, our result significantly improves over previous results \citep{krishnamurthy2018semiparametric}.

\item \textbf{PAC and BAI results for semiparametric bandits.} \;
We provide the first exploration-based guarantees for semiparametric bandits 
(Corollary~\ref{corollary; dimension optimal sample complexity} and Theorem~\ref{Theorem; Exp of PE}) including the PAC bound and BAI guarantee. 
It is important to note that our proposed algorithm simultaneously possesses the properties of low regret, PAC, and BAI.
Our sample complexities match the results for linear bandits studied in \citet{soare2014best}.

\item \textbf{Novel policy design for orthogonalized regression.} \;
Orthogonalized regression is a key technique for semiparametric reward estimation \citep{krishnamurthy2018semiparametric,kim2019contextual,choi2023semi}, 
requiring a G-optimal design step that is typically non-convex and challenging to solve. 
We propose an efficient algorithmic procedure (Theorem~\ref{theorem; optimal variance design}) 
that yields a suitable design solution while remaining computationally tractable, 
thereby enabling effective policy construction in practice.

\item \textbf{Sharper non-asymptotic analysis of orthogonalized regression.} \;
Existing semiparametric bandit approaches often rely on the analysis based on the Cauchy--Schwarz inequality, 
which do not yield the optimal \(\sqrt{d}\) dependence in estimation error. 
We develop a new analysis framework,
achieving dimension-optimal statistical rates for orthogonalized regression (Theorem~\ref{theorem; dimension optimal error bound}) and 
thereby improving upon the standard estimation error bounds, which can be of independent interest.

\end{enumerate}

\subsection{Related Work}

\paragraph{Semiparametric Reward Model in Bandits.}
Several prior studies have explored the semiparametric reward model 
in various settings, including contextual and multi-agent bandit
\citep{greenewald2017action,krishnamurthy2018semiparametric,kim2019contextual,chowdhury2023gbose,choi2023semi}. 
\citet{krishnamurthy2018semiparametric} introduced an algorithm, 
achieving \(\tilde{\mathcal{O}}(d\sqrt{T})\) regret, 
while \citet{kim2019contextual} proposed a Thompson sampling-based method 
with regret \(\tilde{\mathcal{O}}(d^{3/2}\sqrt{T})\). 
\citet{chowdhury2023gbose} then developed a more computationally efficient algorithm 
with \(\tilde{\mathcal{O}}(d\sqrt{T})\) regret, and \citet{choi2023semi} 
studied multi-agent bandits under the same semiparametric framework. 
However, \citet{krishnamurthy2018semiparametric} left open 
the possibility of achieving \(\tilde{\cO}(\sqrt{dT \log K})\) regret. 
All of the prior analyses rely on the Cauchy-Schwarz inequality-based approach, 
which cannot deliver the desired \(\sqrt{d}\)-rate in regret.
Moreover, there has been the absence of exploration-based guarantees 
(such as PAC or BAI) in the semiparametric bandit literature.

\paragraph{Pure Exploration in Linear Bandits.}
\noindent
PAC algorithms in bandits~\citep{sakhi2023pac,chaudhuri2019pac,wagenmaker2022instance} and RL \citep{strehl2009reinforcement} are a key objective, alongside regret minimization, and  have long been a prominent line of research.
Particularly in linear bandits, the PAC and Best Arm Identification (BAI) objectives 
have been thoroughly investigated in works such as 
\citet{soare2014best,degenne2020gamification,li2022instance,jedra2020optimal,fiez2019sequential,komiyama2022minimax, yang2022minimax}, 
where G-optimal design plays a key role in achieving near-optimal sample complexity.
Unlike regret minimization, these methods aim to identify the best action 
or to bound the performance relative to the optimum, 
leading to rich theoretical and practical implications. 
However, in the semiparametric reward model, 
orthogonalized regression is used in place of standard linear regression, 
and these existing design-based strategies for linear bandits 
do not extend to the semiparametric setting. 
As a result, there has been no prior work offering PAC or BAI guarantees 
for semiparametric bandits.

\paragraph{G-optimal Design for Linear Bandits.}
\noindent
Let \(\Xcal = \{x_1, \dots, x_K\} \subset \RR^d\) be a set of features for actions in $\{1, ..., K\}$.
The G-optimal design is defined as the solution to the following optimization problem \citep{LS19bandit-book,soare2014best}: 
\begin{equation} \label{equation; g-optimal design}
v^\star :=  \min_{(p_1, \dots, p_K) \in \Delta^{(K)}} \max_{i \in [K]} \|x_i\|_{\bigl(\sum_i p_ix_ix_i^\top\bigr)^{-1}},
\end{equation}
and a known result gives \(v^\star = \sqrt{d}\).
The problem can be solved efficiently, and the support of the solution has bounded cardinality: \(\bigl|\{i \mid p_i >0\}\bigr| \leq \frac{d(d+1)}{2}\).
The resulting policy minimizes the maximum prediction variance over \(\mathcal{X}\), thereby efficiently solving the PAC and best arm identification problems for linear bandits. 
In semiparametric bandits, however, existing methods rely on 
\emph{orthogonalized regression} (e.g., \citealp{krishnamurthy2018semiparametric}), 
leading to a different, non-convex design objective. 
We formulate a suitable G-optimal design for orthogonalized regression 
and propose an efficient algorithm to solve it, 
thus extending the benefits of experimental designs 
to semiparametric bandits.

\subsection{Notations}\label{subsection; notations}
\noindent
We define \([n] := \{1, 2, \ldots, n\}\) for a positive integer $n$.
We write \(\Ocal(\cdot)\) or \(\lesssim\) to hide absolute constants, and \(\tilde{\Ocal}(\cdot)\) to hide constants and logarithmic factors in \(d\) and \(T\). 
We do \emph{not} hide the factor \(\log K\). 
We use \(a \asymp b\) if \(a \lesssim b\) and \(b \lesssim a\).
We define \(\Delta^{(n)}\) as the \(n\)-dimensional simplex.
For a vector \(x \in \mathbb{R}^d\), we let \(\|x\|_p\) be the \(\ell_p\) norm.
For any positive semidefinite matrix \(\Ab\), we define \(\|x\|_\Ab = \sqrt{x^\top \Ab x}\).
\paragraph{Matrix-inverse-weighted Norm.}
In what follows, we generalize the definition of the matrix-inverse-weighted norm (e.g., \(\| x\|_{\Ab^{-1}}\)) to positive semidefinite matrices, regardless of invertibility. 
To define a matrix-inverse-weighted norm for non-invertible matrices, 
we introduce the following extended definition.
Let \(\Ab\) be a positive semidefinite matrix. We define
\[
\|x\|_{\Ab^{-1}} := \lim_{\lambda \to 0} \|x\|_{(\Ab + \lambda \Ib_d)^{-1}}.
\]
This reduces to the usual definition of \(\|x\|_{\Ab^{-1}}\) when \(\Ab\) is full rank. 
By allowing \(\lambda \to 0\), we extend the definition to handle the non-invertible case 
in a well-defined way.
Even though \(\Ab\) is non-invertible, if \(x\) is contained in the subspace spanned by \(\Ab\)'s eigenvectors, then this value is finite and well-defined.

\section{Preliminaries}\label{section; problem formulation}
\noindent
In this section, we first present our problem setup of semiparametric bandits. Unlike the linear model, our framework includes an additional shift \(\nu_t\), determined prior to action selection, which may be adversarial. 
After that, we introduce orthogonalized regression—a common approach for obtaining the estimator used in all prior works \citep{krishnamurthy2018semiparametric,kim2019contextual,choi2023semi}.

\subsection{Problem Setup: Semiparametric Bandits}\label{subsection; semiparametric bandits with fixed contexts}
\noindent
We consider a finite-armed bandit problem in which each arm \(i\in[K]\) is represented by 
a feature vector \(x_i \in \mathbb{R}^d\). Collectively, these vectors form the set
\(\mathcal{X} = \{x_1, \dots, x_K\}\subset \mathbb{R}^d\). At each round \(t = 1, 2, \ldots, T\),
the learner selects an arm \(a_t \in [K]\) and observes a reward 
\[
r_t \;=\; x_{a_t}^\top \theta^\star \;+\; \nu_t \;+\; \eta_t.
\]
Here, \(\theta^\star \in \mathbb{R}^d\) is an unknown parameter vector, 
\(\nu_t\in\mathbb{R}\) is a bounded shift (which can be adversarially chosen),
and \(\eta_t\) is a noise term. Specifically, we assume that \(\eta_t\)
is independent and sub-Gaussian with variance proxy \(1\) \citep{wainwright2019high}.
We set \(\mathcal{H}_{t}\) as the sigma-algebra generated by
\(\{a_1, r_1, \dots, a_t, r_t\}\). 
We allow the shift \(\nu_t\) to be any \(\mathcal{H}_{t-1}\)-measurable random variable satisfying the boundedness condition in Assumption~\ref{assumption; boundedness}. 
This means \(\nu_t\) may be adversarial, as long as it is determined before the choice of arm at time \(t\).

\paragraph{Feature Span.}
Without loss of generality, we assume that \(\{x_1, \dots, x_K\}\) spans a \(d\)-dimensional subspace of \(\mathbb{R}^d\). 
Should the actual rank be \(d' < d\), we may re-parameterize the problem using 
a \(d'\times d\) matrix \(A\) and a \(d'\times 1\) vector \(\theta^{\star '}\) 
so that for all \(i\in[K]\), \(x_i^\top \theta^\star = (A x_i)^\top \theta^{\star '}\). 
Henceforth, we assume full column rank (i.e.\ rank = \(d\)) without loss of generality.

\paragraph{Optimal Arm and Suboptimality Gap.}
We define the optimal arm as 
\(
a^\star \;:=\; \arg\max_{i\in[K]} \; x_i^\top \theta^\star,
\)
and let 
\(
\Delta_\star \;:=\; x_{a^\star}^\top \theta^\star \;-\; \max_{j \neq a^\star} \, x_j^\top \theta^\star
\)
denote the suboptimality gap. Throughout, we assume the optimal arm is unique, 
though our methods readily extend to settings with multiple optimal arms.

\paragraph{Cumulative Regret.}
Our primary performance criterion for the semiparametric bandit problem is the \emph{cumulative regret}, 
defined by
\[
\Regret(T)
\;:=\;
\sum_{t=1}^T \left( x_{a^\star}^\top \theta^\star \;-\; x_{a_t}^\top \theta^\star \right).
\]
Since \(a^\star\) maximizes the expected reward in the absence of the shift \(\nu_t\), 
the regret measures how much reward is lost by playing suboptimal arms over time.

\paragraph{Fixed Policies and Experimental Design.}
In addition to regret minimization, we consider exploration problems and, more generally, 
estimation tasks that require carefully chosen sampling policies. 
We use the term \emph{design} to denote a (possibly randomized) policy
that chooses arm \(a_t\in[K]\) at each time \(t\). 
Sections~\ref{section; optimal design} and \ref{section; error analysis} focus on 
evaluating the estimation error when a fixed policy (design) is executed, i.e., 
when the policy is determined independently of the data. 
Our aim there is to identify designs that control the prediction variance 
and enable strong PAC or best arm identification guarantees. 
Section~\ref{section; low regret algorithm} then develops an adaptive, data-dependent 
policy that simultaneously achieves low cumulative regret, PAC, and BAI guarantees.

\begin{assumption}[Boundedness of parameter and features]
\label{assumption; boundedness}
We assume \(\|\theta^\star\|_2 \le 1\), \(\|x_i\|_2 \le 1\) for all \(i \in [K]\), 
and \(|\nu_t|\le 1\) for all \(t \ge 1\). 
\end{assumption}
Using $1$ as the bound is without loss of generality and maintains consistency with prior work 
\citep{krishnamurthy2018semiparametric,kim2019contextual,choi2023semi,chowdhury2023gbose}. 
No additional assumption is placed on \(\nu_t\) beyond boundedness; in particular, it may be adversarially chosen, state-dependent, or stochastic, provided it is fixed before each action is selected.

Next, we introduce the orthogonalized regression method for estimating \(\theta^\star\), which is used in all the aforementioned prior works.

\subsection{Estimation of \(\theta^\star\) by Orthogonalized Regression}\label{subsection; orthogonalized regression}
\noindent
For action \(a_s\) sampled at time \(s\) according to a random policy, we define the \emph{centered feature} \(\tilde{x}_{a_s} := x_{a_s} - \EE[x_{a_s} \mid \Hcal_{s-1}]\) so that \(\EE[\tilde{x}_{a_s} \mid \cH_{s-1}] = 0\). 
This expectation is taken with respect to the distribution of \( a_s \), namely, \( \mathbf{p}_s = \{p_{i,s}\}_{i=1}^K \). 
Since \( a_s \) is sampled from a multinomial distribution \( \operatorname{Multinomial}(1, \mathbf{p}_s) \), we have \( \mathbb{E}[x_{a_s} | \mathcal{H}_{s-1}] = \sum_{i=1}^K p_{i,s} x_i \). 
At time \(t\), we define the \emph{centered Gram matrix} \(\widehat{\Vb}_t := \sum_{s=1}^t \tilde{x}_{a_s}\tilde{x}_{a_s}^\top\) and set the \emph{empirical covariance} \(\widehat{\bSigma}_t := \frac{1}{t}\widehat{\Vb}_t\).
We adopt the estimator proposed in \citet{krishnamurthy2018semiparametric}, obtained by regressing the rewards on the centered features via ridge regression as follows,
\begin{equation}\label{equation; estimator}
\hat{\theta}_t = (\widehat{\Vb}_t + \beta_t \Ib_d)^{-1} \sum_{s=1}^t \tilde{x}_{a_s} r_{s},
\end{equation}
where \(\beta_t\) denotes the ridge regularizer.
Roughly speaking, we perform ridge regression on the independent variables \(\{\tilde{x}_{a_s}\}_{s=1}^t\) and the responses \(\{r_s\}_{s=1}^t\).

Consistency of the estimator can be inferred from the following decomposition of the estimation error:
\[
\hat{\theta}_t - \theta^\star
= (\widehat{\Vb}_t + \beta_t \Ib_d)^{-1} \sum_{s=1}^t \tilde{x}_{a_s} \bigl(\EE[x_{a_s} \mid \Hcal_{s-1}]^\top \theta^\star + \nu_s + \eta_s\bigr)
- \beta_t (\widehat{\Vb}_t + \beta_t \Ib_d)^{-1} \theta^\star.
\]
We observe that \(\sum_{s=1}^t \tilde{x}_{a_s}\bigl(\EE[x_{a_s} \mid \Hcal_{s-1}]^\top \theta^\star + \nu_s + \eta_s\bigr)\) is a martingale adapted to the filtration \(\{\mathcal{H}_{s-1}\}_{s=1}^{t}\). 
Therefore, the first term converges to 0 in probability under a suitably well-conditioned centered Gram matrix and covariance, \(\widehat{\Vb}_t\) and \(\widehat{\bSigma}_t\).
Compared to the first term, the second term is negligible and converges to 0 even faster.
Later, we present a sharp analysis of the estimation error in Section~\ref{section; error analysis}, which is one of our main contributions.

\section{Experimental Design for Orthogonalized Regression}\label{section; optimal design}
\noindent
In this section, we discuss efficient experimental design for orthogonalized regression. 
In our work, "design" refers to the policy \(\pb =(p_1, \dots p_K) \in \Delta^{(K)}\), where \(p_i\) denotes the probability of selecting action \(a_i\). 
Specifically, we consider a scenario in which a fixed policy \(\pb\) is employed for pure exploration over a predetermined period. 
Our goal is to identify an effective design for estimation. First, we establish the necessary notations and setup, and then we formulate the optimization problem for the design.

Throughout Sections \ref{section; optimal design} and \ref{section; error analysis}, we consider the case where we pull arms up to time \(t\) according to a fixed policy \(\pb =(p_1, \dots p_K)\) and obtain samples \(\{x_{a_s}, r_s\}_{s=1}^t\).
We then have \(\EE[x_{a_s} \mid \Hcal_{s-1}] = \bar{x}_{\pb}\) and \(\EE[\widehat{\bSigma}_t ] = \EE[\tilde{x}_{a_s}\tilde{x}_{a_s}^\top] := \bSigma_{\pb}\), where \(\bar{x}_{\pb}\) and \(\bSigma_{\pb}\) are defined as follows. 
\begin{definition}[Covariance and mean of policy \(\pb\)]
We define the feature mean of a policy \(\pb\) as 
\[
\bar{x}_{\pb} := \sum_{i=1}^K p_i x_i. 
\]
We also define the covariance of a policy \(\pb\) as
\[
\bSigma_{\pb} = \sum_{i=1}^K p_i\bigl(x_i - \bar{x}_{\pb}\bigr)\bigl(x_i - \bar{x}_{\pb}\bigr)^\top.
\]
\end{definition}
Then we have \(\tilde{x}_{a_s} := x_{a_s} - \bar{x}_{\pb}\) and additionally define \(\Vb_t := \EE[\widehat{\Vb}_t ] = t \bSigma_\pb\).

\paragraph{Ridge Regularizer Selection.}
In Lemma~\ref{lemma; second moment concentration}, we prove that when we choose \(\beta_t = \log(t/\delta)\), we have 
\[
\frac{1}{c}\bigl(\Vb_t + \beta_t \Ib_d \bigr) 
\preceq  \widehat{\Vb}_t + \beta_t \Ib_d  
\preceq c\bigl(\Vb_t + \beta_t \Ib_d \bigr)
\]
with probability at least \(1-\frac{\delta}{10}\) for some absolute constant \(c>0\).
From now on, we use the ridge regularizer \(\beta_t = \log(t/\delta)\) when we perform regression with \(t\) samples.
We also define the \emph{normalized ridge regularizer} \(\lambda_t = \frac{\beta_t}{t}\).
Later, we prove in Appendix~\ref{subsection; computational efficiency} that for our design (presented in Section~\ref{section; optimal design}), the constant $c$ can be $c \in [1,2]$ once $t \gtrsim d \log (dt/\delta)$.

\subsection{Key Quantity to Bound Estimation Error}
\noindent
Consider the situation where we aim to estimate the value of 
\(
  z^\top \theta^\star
\)
at some 
\(z \in \mathbb{R}^d \).
We observe below that \(\| z\|_{\bSigma_{\pb}^{-1}}\) is the key quantity that must be controlled for minimizing the estimation error at \(z \in \RR^d\).
For the extended definition of matrix-inverse-weighted norm (e.g. \(\| \cdot \|_{\bSigma_{\pb}^{-1}}\)), see Section~\ref{subsection; notations}.

To motivate this, we begin by presenting an error analysis based on the Cauchy–Schwarz inequality, as used in all prior works.
Using Lemma 11 of \citet{krishnamurthy2018semiparametric} and Lemma~\ref{lemma; self normalized bound second moment} (modified version), we can derive, for any \(z \in \RR^d\):
\begin{align}\label{equation; cauchy}
\lvert z^\top (\hat{\theta}_t- \theta^\star)\rvert
\leq
\| z\|_{(\Vb_t + \beta_t \Ib_d)^{-1}} \|(\hat{\theta}_t- \theta^\star) \|_{\Vb_t + \beta_t \Ib_d}
\lesssim
\frac{1}{\sqrt{t}} \sqrt{d  \log\bigl(\frac{t}{\delta}\bigr)}\| z\|_{\bSigma_{\pb}^{-1}},
\end{align}
where the first inequality holds via Cauchy–Schwarz. 
It is evident from this inequality that a low value of \(\| z\|_{\bSigma_{\pb}^{-1}}\) guarantees a tight bound on the estimation error.
This result is analogous to linear bandits, except that the covariance of the policy replaces the standard second moment. 

Later in Theorem~\ref{theorem; dimension optimal error bound}, we replace the Cauchy–Schwarz-based error analysis with a sharper inequality that removes the \(\sqrt{d}\) factor from the bound above. (The suboptimality of the Cauchy–Schwarz-based analysis is discussed in Appendix~\ref{section; suboptimality of Cauchy}.) 
The resulting bound is still proportional to \(\| z\|_{\bSigma_{\pb}^{-1}}\). 

To identify the best arm or minimize the regret, it suffices to estimate the expected reward of each arm up to an additive constant, i.e., it suffices to estimate \(x_i^{\top}\theta^{\star}+c\) for some constant \(c\) for all \(i\in [K] \). Letting \(c=-\mu^{\top}\theta^\star\) for some vector \(\mu\in\mathbb{R}^d\), we have \(x_i^\top \theta^{\star}+c=(x_i-\mu)^\top \theta^{\star}\). We therefore aim to find a design that minimizes the maximum value of \(\| z\|_{\bSigma_{\pb}^{-1}}\) over \(z\in\mathcal{X}-\mu:=\{x_1-\mu,\cdots, x_K-\mu\}\), where the value of \(\mu\in\mathbb{R}^d\) is specified in the next section.

\subsection{Main Challenges in Constructing the Experimental Design}
\noindent
We now formulate our optimization problem, aiming to find the optimal design for orthogonalized linear regression.

\paragraph{Parallel Shifting of Features.}

A natural choice of $\mu$ is $\mu=0$. However, we find that it can sometimes be impossible to finitely bound \(\max_{i \in [K]}\| x_i\|_{\bSigma_{\pb}^{-1}}\). This is because the subspace of \(\bSigma_{\pb}\) is spanned by \(\{x_i - x_1\}_{i \in [K]}\) (by Lemma~\ref{lemma; covariance decomposition}), which can be a strict subspace of \(\operatorname{span}(x_1,\dots,x_K)\). This problem can be circumvented by shifting the contexts using an appropriate nonzero vector $\mu$. Based on this observation, we present our G-optimal design problem for orthogonalized regression.

\begin{definition}[G-optimal design of orthogonalized regression]\label{problem; orthogonalized design}
We aim to find a policy $\pb$ that achieves
\[
\min_{\mu \in \RR^d, \, \pb \in \Delta^{(K)}} \, \max_{i \in [K]} \| x_i -\mu \|_{\bSigma_{\pb}^{-1}},
\]
where
\begin{align*}
    \bSigma_{\pb}= \sum_{i=1}^K p_i (x_i - \bar{x}_{\pb})(x_i - \bar{x}_{\pb})^\top  \,\, \text{  and  }\,\, \bar{x}_\pb  = \sum_{i=1}^K p_i x_i.
\end{align*}
\end{definition}

\paragraph{Challenges and Differences from Linear Bandits.}
Our optimization problem includes the covariance term \(\sum_{i=1}^K p_i (x_i - \bar{x}_{\pb})(x_i - \bar{x}_{\pb})^\top\), which is a \textbf{third-order} polynomial of \(p_i\) and hence is \textbf{non-convex}. In contrast, the optimal design for linear bandits is formulated with the second moment \(\sum_{i=1}^K p_ix_ix_i^\top\), which is linear in \(p_i\). The non-convexity renders our problem much more challenging, calling for non-convex optimization methods.

\paragraph{Our Goal.}
For linear bandits, the G-optimal design \cref{equation; g-optimal design} always has optimal cost $v^\star = \sqrt{d}$.
Since our reward model includes the linear model as a special case (\(\nu_t=0\)), if we solve our optimization problem and obtain \(\Ocal(\sqrt{d})\), we can view it as a sufficiently "nice" solution. We propose a design that achieves $\mathcal{O}(\sqrt{d})$, leading to an optimal regret bound and tight PAC and BAI results.

\subsection{Proposed Experimental Design and Performance Guarantees}
\noindent
Due to non-convexity, the optimization problem in Definition \ref{problem; orthogonalized design} is hard to solve exactly.
In this Section, instead of solving the problem exactly, we propose an algorithm which achieves $\mathcal{O}(\sqrt{d})$, which is sufficiently good for the estimation, which we will discuss more in next Section~\ref{section; error analysis}.

There are two main challenges. The first is handling the covariance of the policy, 
which involves third-order terms of \(\pb\). The second is choosing \(\mu\). 
We need to pick an appropriate \(\mu\) that effectively minimizes 
Problem~\ref{problem; orthogonalized design}.

Our algorithm is surprisingly simple, yet it handles the non-convex optimization problem efficiently.
We define \(b_i = x_i - x_1\) for \(i = 2, \dots, K\).
First, we find the G-optimal design over \(b_2, \dots, b_K\) for standard linear regression (as defined in \cref{equation; g-optimal design}) and denote it by \((\tilde{p}_2, \dots, \tilde{p}_K)\).
Then we return our final policy \(\pb =\bigl(\frac{1}{2}, \frac{\tilde{p}_2}{2},\dots, \frac{\tilde{p}_K}{2}\bigr)\).
Our optimal design algorithm, named \algoa, is presented in Algorithm~\ref{algorithm; ED}.

\begin{algorithm}[]
\caption{\texttt{DEO}:  \textbf{D}esign of \textbf{E}xperiment for \textbf{O}rthogonalized Regression}
\label{algorithm; ED}
\begin{algorithmic}
\Require Feature set $\cX$ and $\delta>0$.
\State Calculate \(b_2 = x_2 - x_1, \dots, b_K = x_K - x_1\).
\State Find the G-optimal design \cref{equation; g-optimal design} for \(b_2, \dots, b_K\), and set the obtained policy as \(\Tilde{p}_2, \dots, \Tilde{p}_K\).
\State Set \(p_1 = \frac{1}{2} \) and 
\( p_i = \frac{\Tilde{p}_i}{2}\) for $i \geq 2$.\\
\Return \(\pb^{\des} = (p_1, \dots, p_K)\).
\end{algorithmic}
\end{algorithm}
\noindent
We denote the design of \algoa\ as \(\pb^{\des} = (p_1^{\des}, \dots, p_K^{\des})\) and redefine the covariance of \(\pb^{\des}\) as \(\bSigma_{\des}\).
We also redefine the feature mean of \(\pb^{\des}\) as \(\bar{x}_{\pb^{\des}} := \bar{x}_{\des}\).

\begin{theorem}[Performance of Algorithm~\ref{algorithm; ED}]\label{theorem; optimal variance design}
Our policy obtained by Algorithm~\ref{algorithm; ED} satisfies
\[
\|x_i - x_1\|_{\bSigma_{\des}^{-1}} \leq 2\sqrt{d},
\]
for all \(i \in [K]\). 
Also, for all \(i \in [K]\),
\begin{align*}
\|x_i - \bar{x}_{\des}\|_{\bSigma_{\des}^{-1}} \leq 4\sqrt{d}.
\end{align*} 
Additionally, the support of the policy satisfies
\(
\bigl|\{i \in [K]\colon p_i^{\des} > 0\}\bigr| \leq \frac{d(d+1)}{2}.
\)
\end{theorem}

\paragraph{Discussion of Theorem~\ref{theorem; optimal variance design}.}
This theorem shows that our policy obtained by Algorithm~\ref{algorithm; ED} effectively solves the main optimization problem of Problem~\ref{problem; orthogonalized design} and achieves a result up to a constant factor of the golden value, \(\sqrt{d}\).
By simply selecting \(\mu = x_1\) in Problem~\ref{problem; orthogonalized design}, it achieves a 2-approximation with respect to our benchmark quantity \(\sqrt{d}\).
Implementing this design, we can efficiently estimate \(\bigl(x_i - x_1\bigr)^\top \theta^\star\) for all \(i \in [K]\), which is sufficient for obtaining optimal regret bound.
Moreover, even if we choose \(\mu = \bar{x}_{\des}\), we still attain performance on the order of \(\cO(\sqrt{d})\).
Its proof is deferred to Appendix~\ref{section; proof design}.

\section{Analysis of Estimation Error for Orthogonalized Regression}\label{section; error analysis}

\noindent
Previously, we obtained a policy design that successfully bound \(\max_{x \in \cX} \|x-\mu\|_{\bSigma_\pb^{-1}}\lesssim \sqrt{d}\) for \(\mu = x_1\). The next goal is to develop a sharp estimation error bound for orthogonalized regression under the fixed policy. All of the previous studies \citet{krishnamurthy2018semiparametric,kim2019contextual,choi2023semi} obtained error bounds 
using a Cauchy--Schwarz-based analysis. 
However from \cref{equation; cauchy}, we see that this approach leads to a estimation error of order \(\mathcal{O}({d\over \sqrt{t}})\) which is suboptimal in \(d\). 
We discuss this further in Appendix~\ref{section; suboptimality of Cauchy}.
Even in linear bandits, sharp results in dimension $d$ are never obtained via a Cauchy–Schwarz-based analysis \citep{LS19bandit-book}.
In this Section, we provide a novel non-asymptotic estimation error analysis for orthogonalized regression, 
which yields an error bound of order \(\mathcal{O}({\sqrt{d}\over\sqrt{t}})\). 
This rate is optimal for linear reward models. 
Since the linear model is a special case of our setup, the rate cannot be further improved.

\subsection{Novel and Sharp Estimation Error Analysis}
\noindent
We present our novel error analysis, which does not use a Cauchy--Schwarz-based approach.
Because of orthogonalized regression, we devise entirely new techniques for bounding the estimation error.
\begin{theorem}[Estimation error upper bound]\label{theorem; dimension optimal error bound}
Suppose we obtain \(t\) samples \(\{x_{a_s}, r_s\}_{s=1}^t\) from pure exploration with a fixed policy \(\pb\).
For some \(z \in \RR^d\), set \(\| z \|^2_{\bSigma_\pb^{-1}} = L\) and \( \max_{i \in [K]} \| x_i - \bar{x}_\pb\|^2_{\bSigma_\pb^{-1}} = M\).
Then the estimator \(\hat{\theta}_t\) obtained by orthogonalized regression with ridge regularizer \(\beta_t =\log(t/\delta)\) satisfies
\[
\lvert z^\top (\hat{\theta}_t-\theta^\star) \rvert \leq C_1 \biggl(\frac{\sqrt{ L \log\bigl(\frac{t}{\delta}\bigr)}}{\sqrt{t}} + \frac{\sqrt{L}M\log\bigl(\frac{d}{\delta}\bigr)}{t} \biggr)
\]
with probability at least \(1-\delta\) for a universal constant \(C_1>0\).
Also, the leading term \(\frac{\sqrt{L}}{\sqrt{t}}\) matches the lower bound up to some constant and cannot be improved.
\end{theorem}

\paragraph{Discussion of Theorem~\ref{theorem; dimension optimal error bound}.}
This theorem provides a novel and sharp non-asymptotic error bound for orthogonalized regression with an improved rate.
When we use our experimental design \(\pb\) from \algoa\ and choose \(z = x_i -x_1\), we have \(L,M \lesssim d\) by Theorem~\ref{theorem; optimal variance design}, leading to an estimation error bound of \(\tilde{\Ocal}\bigl(\frac{\sqrt{d \log K}}{\sqrt{t}}\bigr)\), which is on the \(\sqrt{d}\) scale. 
It matches the minimax rate of linear regression with dimension \(d\), and since our model is broader and more challenging than the standard linear model, it also matches the minimax rate in this context. 
To the best of our knowledge, this is the first dimension-optimal result (\(\sqrt{d}\)-rate upper bound) for the non-asymptotic analysis of orthogonalized regression.
A Cauchy--Schwarz-based analysis, which is used in all previous literature on orthogonalized regression, cannot meet this rate, making this a significant improvement.
Its proof is deferred to Appendix~\ref{section; proof error analysis}.

\subsection{Warm-up: Pure Exploration with \algoa}\label{subsection; pure exploration}
\noindent
We now state our pure exploration strategy using the policy design \algoa.
The procedure is simple: 
\begin{enumerate}
\item Find a policy \(\pb^{\des}\) by running \algoa.
\item Sample actions with \(\pb^{\des}\) for \(t\) rounds and then stop. 
\item The output is a greedy policy $a_t= \arg \max_i x_i^\top \hat{\theta}_t $, where $\hat{\theta}_t$ is our estimator (\cref{equation; estimator}) with regularizer $\beta_t> 0$.
\end{enumerate}
This is a warm-up version of pure exploration; in the next section, we propose an algorithm that achieves low regret, PAC guarantees, and also performs BAI.

For any \(\varepsilon, \delta> 0\), an \((\varepsilon, \delta)\)-PAC algorithm aims to ensure that the value function is close to the optimal value.
It aims to find a policy \(\pi\) satisfying \(V^\pi \geq V^\star - \varepsilon\) with probability at least \(1-\delta\), where \(V^\star\) is the value of the optimal policy and \(V^\pi\) is the value of policy \(\pi\). 
We define the sample complexity required to achieve this as \(\bm{\tau}(\varepsilon, \delta)\), which we refer to as the \((\varepsilon, \delta)\)-PAC bound.
We next present the PAC bound of our pure exploration strategy.

\begin{corollary}[PAC bound]\label{corollary; dimension optimal sample complexity}
Suppose we are conducting pure exploration with policy \(\pb^{\des}\) and $\beta_t = \log(t/\delta)$ as described above.
Its \((\varepsilon, \delta)\)-PAC bound satisfies
\[
\bm\tau(\varepsilon, \delta) \geq  C_2 \Bigl(\frac{d\log\bigl(\frac{dK}{\varepsilon \delta}\bigr)}{\varepsilon^2} +\frac{d^{\frac{3}{2}}\log\bigl(\frac{dK}{\delta}\bigr)}{\varepsilon} \Bigr)
= \tilde{\Ocal}\Bigl(\frac{d}{\varepsilon^2} \log K \Bigr)
\]
for some absolute constant \(C_2> 0\).
\end{corollary} 

\paragraph{Discussion of Corollary~\ref{corollary; dimension optimal sample complexity}.}
This is the first result on pure exploration and PAC bound for the class of bandit problems with a semiparametric reward model. 
The proof is deferred to Appendix~\ref{section; proof sample complexity}.
By using an arm elimination technique, with a simple modification, we can obtain a BAI strategy with sample complexity \(\tilde{\Ocal}\bigl(d \log K / \Delta_\star^2\bigr)\).
Later, in Section~\ref{section; low regret algorithm}, we show that our low-regret algorithm (Algorithm~\ref{algorithm; PE}) also achieves BAI while enjoying PAC guarantees.

\section{Main Algorithm with Low Regret, PAC and BAI Properties}\label{section; low regret algorithm}
\noindent
We now present our main algorithm. 
This algorithm exhibits an instance-independent regret of order \(\tilde{\Ocal}(\sqrt{dT \log K})\), which is the first optimal result in bandit problems with a semiparametric reward model. We also derive a problem-dependent regret bound of order \(\Tilde \cO(d \log K /\Delta_{\star})\). 
Furthermore, our algorithm simultaneously achieves BAI with sample complexity \(\Tilde \cO(d \log K /\Delta_{\star}^2)\), and enjoys an \((\epsilon,\delta)\)-PAC guarantee with sample complexity \(\tilde \cO (d \log K/\varepsilon^2)\).

\subsection{Proposed Algorithm}
\noindent
Our algorithm, \algob, is shown in Algorithm~\ref{algorithm; PE}. 
\algob\ adopts the phase-elimination scheme and incorporates our experimental design \algoa.
Given a G-optimal design,
\citet{LS19bandit-book} studied a phase elimination scheme that achieves $\tilde{\mathcal{O}}(\sqrt{dT})$ regret for linear reward model.
By combining this scheme with the results of Theorems~\ref{theorem; optimal variance design} and \ref{theorem; dimension optimal error bound}, we propose the following low-regret algorithm.
At the beginning of the \(\ell\)-th phase, we denote the set of arms that were not eliminated up to the previous phase as \(\cA_\ell\). We compute the policy \(\pb^{\des}_\ell\) over \(\cA_\ell\) via Algorithm~\ref{algorithm; ED}, wherein the role of $x_1$ is replaced by $\cA_\ell(1)$, the arm in \(\cA_\ell\) with the smallest index.
Then sample actions according to \(\pb^{\des}_\ell\) for
\begin{align}\label{equation; n_ell}
n_\ell := 4 C_2\left\lceil 
\frac{d}{\varepsilon_\ell^2}\log\bigl(\frac{dK\ell(\ell+1)}{\delta\varepsilon_\ell}\bigr) +\frac{d^{\frac{3}{2}}}{\varepsilon_\ell}\log\bigl(\frac{dK \ell (\ell+1)}{\delta}\bigr) \right\rceil
\end{align}
times. 
At the end of the phase, we calculate \(\hat{\theta}_{(\ell)}\) using orthogonalized regression with ridge regularizer \(\beta_{(\ell)} = \log\bigl(\frac{n_\ell\ell(\ell+1)}{\delta}\bigr)\) on the samples obtained during phase \(\ell\). 
We then eliminate arms from \(\cA_\ell\) whose estimated rewards are less than the maximum by more than \(\varepsilon_\ell\).
The aforementioned procedure is repeated for \(\ell=1,2,\cdots\) until only one arm survives.
If there is only one arm left, declare it as the best arm and select that arm until the end.

\begin{algorithm}[]
\caption{\texttt{SBE}: \textbf{S}emiparametric \textbf{B}andits with \textbf{E}limination}
\label{algorithm; PE}
\textbf{Input: } Features \(\cX\), \(\delta>0\) and \(\Bb_0 = 0 \Ib_d, \bb_0 =\bm{0}\). \\
\textbf{Initialize} \(\cA_1 = [K]\), \(t=1\).\\
\For{\(\ell = 1, 2, \dots\)}{
\If {\(|\cA_{\ell}| = 1\)}{declare the arm in \(\cA_{\ell}\) as the best arm and select that arm until the end.}
\Else{
Set \(\varepsilon_\ell = \frac{1}{2^\ell}\) and calculate a policy \(\pb^{\des}_\ell\) using \algoa\ for the remaining arms \(\cA_\ell\). 
The role of $x_1$ in \algoa\ is replaced by $\cA_\ell(1)$.

Set \(n_\ell\) from \cref{equation; n_ell}.

\For{\(j=1,2, \dots, n_\ell\)}{
Pull an arm according to \(\pb^{\des}_\ell\). Let the sampled arm be \(a_t\), and receive reward \(r_t\).

Set \(\tilde{x}_{a_t} = x_{a_t} - \EE[x_{a_t} \mid \cH_{t-1}]\). 
Update \(\Bb_j= \Bb_{j-1} + \tilde{x}_{a_t}\tilde{x}_{a_t}^\top\), \(\bb_j = \bb_{j-1} + \tilde{x}_{a_t}r_t\).

Update \(t \leftarrow t+1\). If \(t+1 \geq T\), exit.
}
Calculate \(\hat{\theta}_{(\ell)} = (\Bb_{n_\ell} + \beta_{(\ell)} \Ib_d)^{-1} \bb_{n_\ell}\) for \(\beta_{(\ell)} = \log\bigl(n_\ell \ell(\ell+1) /\delta\bigr)\).\\
Eliminate arms
\(
\bigl\{ x \in \cA_{\ell} \big|  \max_{x' \in \cA_{\ell}} (x')^\top \hat{\theta}_{(\ell)} - x^\top \hat{\theta}_{(\ell)}  > \varepsilon_\ell \bigr\}
\)
and update the remaining arm set to \(\cA_{\ell+1}\).

Reset \(\Bb_0 = 0 \cdot \Ib_d\), \(\bb_0 = 0_d\).
}
}
\end{algorithm}

\subsection{Regret Analysis}
\noindent
First, we present the regret bound of \algob\ without any dependency on the suboptimality gap. 
Our result is a \(\tilde{\Ocal}(\sqrt{dT\log K})\) cumulative regret, which is the first result for bandits with a semiparametric reward model.

\begin{theorem}[Regret bound of \algob]\label{theorem; regret bound without gap}
The \algob\ algorithm has the following cumulative regret bound with probability at least \(1-\delta\):
\[
\Regret(T) \lesssim \sqrt{dT \log\bigl(K/\delta\bigr)} + \sqrt{dT}\log T + d^{3/2}\log\left(\frac{dKT}{\delta}\right)\log(\frac{T}{d})
= \tilde{\Ocal}\bigl(\sqrt{dT\log K}\bigr).
\]
\end{theorem}

\paragraph{Discussion of Theorem~\ref{theorem; regret bound without gap}.}
The bound matches known optimal results for linear bandits (up to logarithmic factors), even though our problem is more challenging. 
Compared to known results, this is the first \(\sqrt{d}\)-rate regret bound for bandits with a semiparametric reward model.
With small modifications, one obtains
\(\min\bigl(\tilde{\Ocal}(\sqrt{dT \log K}), \tilde{\Ocal}(d\sqrt{T})\bigr)\)
as shown in Appendix~\labelcref{section; K-independent results,section; proof K independent results}.
Hence, with a reasonably finite number of arms, our result provides the sharpest regret bound among bandits with a semiparametric reward model.
Its proof is deferred to Appendix~\ref{section; proof main algorithm}.

Next, we present a gap-dependent regret bound that achieves a logarithmic scale in the time horizon \(T\).

\begin{theorem}[Gap-dependent regret bound of \algob]\label{theorem; regret bound with gap}
The \algob\ algorithm has the following cumulative regret bound with probability at least \(1-\delta\):
\[
\Regret(T) \lesssim \Bigl(\frac{d}{\Delta_\star} + d^{3/2}\Bigr)\log \Bigl(\frac{dK}{\delta \Delta_\star}\Bigr)\log(\frac{1}{\Delta_\star})
= \tilde{\Ocal}\Bigl(\frac{d}{\Delta_\star} \log K\Bigr).
\]
\end{theorem}

\paragraph{Discussion of Theorem~\ref{theorem; regret bound with gap}.}
This result shows logarithmic regret, and we highlight that it is the first such result in bandits with a semiparametric reward model.
Similarly to the high-probability regret, one can derive the expected regret of order \(\Ocal\bigl(\frac{d}{\Delta_\star} \log(\frac{TdK}{\Delta_\star})\bigr)\) by setting $\delta = \frac{1}{T}$.
Hence, when the suboptimality gap \(\Delta_\star\) is reasonably small, our result improves upon the existing algorithms \citet{greenewald2017action,krishnamurthy2018semiparametric,kim2019contextual}.
The proof is in Appendix~\ref{section; proof main algorithm}.

\subsection{PAC and BAI Properties}
\noindent 
Even though our algorithm has sublinear regret, it possesses exploration-based properties such as PAC and BAI.
Algorithm~\ref{algorithm; PE} declares the remaining action as the best arm when only one action is left.  
We prove that the declared arm is indeed the best arm with high probability and provide its sample complexity.

\begin{theorem}[PAC and BAI properties of \algob]\label{Theorem; Exp of PE}
Our \algob\ enjoys both PAC and BAI properties, as follows:\\
\textbf{(BAI): }
Our \algob\ algorithm outputs the best arm with probability at least \(1-\delta\) upon selecting \(\bm\tau_{BAI}\) samples, where
\[
\bm\tau_{BAI} = \tilde{\Ocal}\Bigl(\frac{d}{\Delta_\star^2}\log K\Bigr).
\]\\
\textbf{(PAC): }
At time \(t\), let the policy of \algob\ be \(\pi_t\).
with probability at least \(1-\delta\), we have \(V^{\pi_t} \geq V^\star - \varepsilon\) whenever \(t \geq \bm \tau(\varepsilon, \delta)\) for
\[
\bm \tau(\varepsilon, \delta) = \tilde{\Ocal}\Bigl(\frac{d}{\varepsilon^2} \log K\Bigr).
\]
\end{theorem}

\paragraph{Discussion of Theorem~\labelcref{Theorem; Exp of PE}.}
These corollaries show that \algob\ achieves low regret and also performs BAI while satisfying the PAC property. 
Previously proposed low-regret algorithms in \citet{kim2019contextual,krishnamurthy2018semiparametric} do not have BAI and PAC guarantees as stated in Theorem~\labelcref{Theorem; Exp of PE}. Hence, the proposed algorithm \algob\ is the first algorithm that comes with such guarantees while also achieving near-optimal regret, which we believe is a major contribution to the literature.

However, the BAI sample complexity presented above is not minimax optimal when the instance is fixed.  
In the case of linear bandits, the instance-dependent minimax sample complexity can be smaller than $\Omega(\frac{d}{\Delta_\star^2})$, as studied in many literature \citep{jedra2020optimal, fiez2019sequential}.  
Therefore, this result can be practically and theoretically suboptimal, and improving it is considered a promising future direction.

\subsection{Application to MAB with Semiparametric Rewards}\label{section; application to MAB}
\noindent
Furthermore, we would like to mention that our results are also applicable to the MAB with a semiparametric reward model. The semiparametric MAB assumes the following model: there are $K$ arms and each arm $i \in [K]$ has reward distribution of mean $\mu_i \in \RR$.
When arm $i$ is pulled at time $t$, a shifted reward $r_t = \mu_i + \varepsilon_t + \nu_t$ is received. Here, $\varepsilon_t$ is a mean-zero sub-Gaussian noise, and $\nu_t$ is a shift determined before the arm selection.
This can be incorporated into our setup by setting the features as standard basis of $\RR^K$.
In this case, $d=K$. Directly applying the results from previous work \citep{kim2019contextual,krishnamurthy2018semiparametric} yields a regret bound of $\tilde{\mathcal{O}}(K\sqrt{T})$, which is suboptimal. Our algorithm and results attain a regret bound of $\tilde{\mathcal{O}}(\sqrt{KT})$ (also hiding $\log K$ terms), which matches the known lower bound.
It is also important to note that an instance-dependent logarithmic regret is achieved. Our results guarantee a regret bound of $\tilde{\mathcal{O}}(K/\Delta_\star)$, which, while not matching the known MAB lower bound, is still favorable.

\section{High-level Proof Sketch}
\noindent
The core of our results lies in Theorem~\ref{theorem; optimal variance design} and Theorem~\ref{theorem; dimension optimal error bound}. Once these two theorems are established, other results can be obtained by applying standard techniques used in the linear bandit literature.

Theorem~\ref{theorem; optimal variance design} is obtained through Lemma~\ref{lemma; covariance decomposition}.
Let $\tilde{p}_2, \dots \tilde{p}_K$ be the G-optimal design computed for the features $\{x_2-x_1, \dots x_K-x_1\}$, and let $\bSigma_{\operatorname{opt},1}$ be the second moment of $\{x_2-x_1, \dots x_K-x_1\}$ under that policy (G-optimal design).
By Lemma~\ref{lemma; covariance decomposition}, our second moment satisfies $\bSigma_{\operatorname{deo}} \succeq \frac{1}{4} \bSigma_{\operatorname{opt},1}$, which completes the proof.
A detailed proof is provided in Appendix~\ref{section; proof design}.

The key to the proof of Theorem~\ref{theorem; dimension optimal error bound} is performing decorrelation. The estimation error $\hat{\theta}_t -\theta^\star$
is decomposed as follows; see Appendix~\ref{section; proof error analysis} for detailed notation:
\begin{align*}
     \hat{\theta}_t - \theta^\star =\underbrace{(\widehat{\Vb}_t + \beta_t  \Ib_d)^{-1} \sum_{s=1}^t \tilde{x}_{a_s}\underbrace{\bigl(\bar{x}^\top \theta^\star + \nu_s\bigr)}_{:=q_s}}_{:=\Acr} 
+ \underbrace{(\widehat{\Vb}_t + \beta_t  \Ib_d)^{-1} \sum_{s=1}^t \tilde{x}_{a_s}\eta_s}_{:=\Bcr} 
- \underbrace{\beta_t (\widehat{\Vb}_t + \beta_t  \Ib_d)^{-1}\theta^\star}_{:=\Ccr}.
\end{align*}
Among these, $\Bcr$ and $\Ccr$ are terms that also appear in the error decomposition of standard linear regression and are easily controlled. 
The problematic term is $\Acr$. 
The difficulty in the analysis arises because the randomness of the mean-zero vector $\tilde{x}_{a_s}$ is correlated with $(\hat{\Vb}_t + \beta_t \Ib_d)^{-1}$.
First, we define $e_s:= (\bSigma + \lambda_t \Ib_d)^{-1}\tilde{x}_{a_s} $, which is a mean-zero vector.
By manipulating the expression, for any $z \in \RR^d$, we can decompose $z^\top \Acr$ as follows:
\begin{align*}
     z^\top \Acr = \frac{1}{t}z^\top \bigl(\bSigma + \lambda_t  \Ib_d\bigr)^{-\frac{1}{2}} \sum_{s=1}^t e_s q_s 
+ \frac{1}{t}z^\top \bigl(\widehat{\bSigma}_t + \lambda_t  \Ib_d\bigr)^{-1} \bigl(\bSigma - \widehat{\bSigma}_t\bigr)\bigl(\bSigma + \lambda_t  \Ib_d\bigr)^{-\frac{1}{2}} \sum_{s=1}^t e_s q_s 
\end{align*}
The first term on the right-hand side can be analyzed because the matrix inverse term ($ (\bSigma + \lambda_t  \Ib_d\bigr)^{-\frac{1}{2}}$) and the martingale term ($\sum_{s=1}^t e_s q_s $) are decorrelated, and the second term becomes a higher-order term of order $\tilde{\cO}(\frac{1}{t})$.
A detailed proof can be found in Appendix~\ref{section; proof error analysis}.

\section{Conclusion}
\noindent
We introduced the first experimental design framework for semiparametric bandits, 
yielding both pure exploration guarantees (PAC and BAI) and tight regret bounds, 
including the \(\tilde{\cO}(\sqrt{dT \log K})\) rate and its gap-dependent logarithmic counterpart. 
Our work opens several avenues for further study:  
(1) refining pure exploration algorithms to achieve optimal sample complexities (which may not simultaneously achieve optimal regret), 
and (2) developing methods for the fixed-budget setting, 
a natural counterpart to fixed-confidence pure exploration in linear bandits.

\section*{Acknowledgements}
\noindent
We are very grateful for the constructive comments provided by the reviewers during our COLT submission. 
Gi-Soo Kim was supported by the Institute of Information \&
Communications Technology Planning \& Evaluation~(IITP) grants funded by
the Korea government~(MSIT) (No. RS-2020-II201336, Artificial Intelligence Graduate School Program (UNIST); No. 2022-0-00469, Development of Core Technologies for Task-oriented Reinforcement
Learning for Commercialization of Autonomous Drones).
Min-hwan Oh was supported by the National Research Foundation of Korea~(NRF) grant funded by the Korea government~(MSIT) (No.  RS-2022-NR071853 and RS-2023-00222663) and by Institute of Information \& communications Technology Planning \& Evaluation~(IITP) grant funded by the Korea government~(MSIT) (No.RS-2025-02263754).

\clearpage
\newpage
\bibliography{ref}

\begin{thebibliography}{28}
\providecommand{\natexlab}[1]{#1}
\providecommand{\url}[1]{\texttt{#1}}
\expandafter\ifx\csname urlstyle\endcsname\relax
  \providecommand{\doi}[1]{doi: #1}\else
  \providecommand{\doi}{doi: \begingroup \urlstyle{rm}\Url}\fi

\bibitem[Abbasi-Yadkori et~al.(2011)Abbasi-Yadkori, P{\'a}l, and
  Szepesv{\'a}ri]{abbasi2011improved}
Yasin Abbasi-Yadkori, D{\'a}vid P{\'a}l, and Csaba Szepesv{\'a}ri.
\newblock Improved algorithms for linear stochastic bandits.
\newblock \emph{Advances in neural information processing systems}, 24, 2011.

\bibitem[Boucheron et~al.(2013)Boucheron, Lugosi, and
  Massart]{boucheron2013concentration}
St{\'e}phane Boucheron, G{\'a}bor Lugosi, and Pascal Massart.
\newblock \emph{Concentration inequalities: A nonasymptotic theory of
  independence}.
\newblock Oxford university press, 2013.

\bibitem[Chaudhuri and Kalyanakrishnan(2019)]{chaudhuri2019pac}
Arghya~Roy Chaudhuri and Shivaram Kalyanakrishnan.
\newblock Pac identification of many good arms in stochastic multi-armed
  bandits.
\newblock In \emph{International Conference on Machine Learning}, pages
  991--1000. PMLR, 2019.

\bibitem[Choi et~al.(2023)Choi, Kim, Paik, and Paik]{choi2023semi}
Young-Geun Choi, Gi-Soo Kim, Seunghoon Paik, and Myunghee~Cho Paik.
\newblock Semi-parametric contextual bandits with graph-laplacian
  regularization.
\newblock \emph{Information Sciences}, 645:\penalty0 119367, 2023.

\bibitem[Chowdhury et~al.(2023)Chowdhury, Ismayilzada, Sayem, and
  Kim]{chowdhury2023gbose}
Mubarrat Chowdhury, Elkhan Ismayilzada, Khalequzzaman Sayem, and Gi-Soo Kim.
\newblock Gbose: Generalized bandit orthogonalized semiparametric estimation.
\newblock \emph{arXiv preprint arXiv:2301.08781}, 2023.

\bibitem[Degenne et~al.(2020)Degenne, M{\'e}nard, Shang, and
  Valko]{degenne2020gamification}
R{\'e}my Degenne, Pierre M{\'e}nard, Xuedong Shang, and Michal Valko.
\newblock Gamification of pure exploration for linear bandits.
\newblock In \emph{International Conference on Machine Learning}, pages
  2432--2442. PMLR, 2020.

\bibitem[Fan and Wang(2019)]{fan2019bernstein}
Xiequan Fan and Shen Wang.
\newblock Bernstein type inequalities for self-normalized martingales with
  applications.
\newblock \emph{Statistics}, 53\penalty0 (2):\penalty0 245--260, 2019.

\bibitem[Fiez et~al.(2019)Fiez, Jain, Jamieson, and
  Ratliff]{fiez2019sequential}
Tanner Fiez, Lalit Jain, Kevin~G Jamieson, and Lillian Ratliff.
\newblock Sequential experimental design for transductive linear bandits.
\newblock \emph{Advances in neural information processing systems}, 32, 2019.

\bibitem[Ge et~al.(2023)Ge, Tang, Fan, Ma, and Jin]{ge2023maximum}
Jiawei Ge, Shange Tang, Jianqing Fan, Cong Ma, and Chi Jin.
\newblock Maximum likelihood estimation is all you need for well-specified
  covariate shift.
\newblock \emph{arXiv preprint arXiv:2311.15961}, 2023.

\bibitem[Greenewald et~al.(2017)Greenewald, Tewari, Murphy, and
  Klasnja]{greenewald2017action}
Kristjan Greenewald, Ambuj Tewari, Susan Murphy, and Predag Klasnja.
\newblock Action centered contextual bandits.
\newblock \emph{Advances in neural information processing systems}, 30, 2017.

\bibitem[Jedra and Proutiere(2020)]{jedra2020optimal}
Yassir Jedra and Alexandre Proutiere.
\newblock Optimal best-arm identification in linear bandits.
\newblock \emph{Advances in Neural Information Processing Systems},
  33:\penalty0 10007--10017, 2020.

\bibitem[Kazerouni and Wein(2021)]{kazerouni2021best}
Abbas Kazerouni and Lawrence~M Wein.
\newblock Best arm identification in generalized linear bandits.
\newblock \emph{Operations Research Letters}, 49\penalty0 (3):\penalty0
  365--371, 2021.

\bibitem[Kim and Paik(2019)]{kim2019contextual}
Gi-Soo Kim and Myunghee~Cho Paik.
\newblock Contextual multi-armed bandit algorithm for semiparametric reward
  model.
\newblock In \emph{International Conference on Machine Learning}, pages
  3389--3397. PMLR, 2019.

\bibitem[Kim et~al.(2021)Kim, Kim, and Paik]{kim2021doubly}
Wonyoung Kim, Gi-soo Kim, and Myunghee~Cho Paik.
\newblock Doubly robust thompson sampling for linear payoffs.
\newblock In \emph{Advances in neural information processing systems}, 2021.

\bibitem[Komiyama et~al.(2022)Komiyama, Tsuchiya, and
  Honda]{komiyama2022minimax}
Junpei Komiyama, Taira Tsuchiya, and Junya Honda.
\newblock Minimax optimal algorithms for fixed-budget best arm identification.
\newblock \emph{Advances in Neural Information Processing Systems},
  35:\penalty0 10393--10404, 2022.

\bibitem[Krishnamurthy et~al.(2018)Krishnamurthy, Wu, and
  Syrgkanis]{krishnamurthy2018semiparametric}
Akshay Krishnamurthy, Zhiwei~Steven Wu, and Vasilis Syrgkanis.
\newblock Semiparametric contextual bandits.
\newblock In \emph{International Conference on Machine Learning}, pages
  2776--2785. PMLR, 2018.

\bibitem[Lattimore and Szepesv\'{a}ri(2019)]{LS19bandit-book}
Tor Lattimore and Csaba Szepesv\'{a}ri.
\newblock \emph{Bandit Algorithms}.
\newblock Cambridge University Press (preprint), 2019.

\bibitem[Li et~al.(2019)Li, Wang, and Zhou]{li2019nearly}
Yingkai Li, Yining Wang, and Yuan Zhou.
\newblock Nearly minimax-optimal regret for linearly parameterized bandits.
\newblock In \emph{Conference on Learning Theory}, pages 2173--2174. PMLR,
  2019.

\bibitem[Li et~al.(2022)Li, Ratliff, Jamieson, Jain, et~al.]{li2022instance}
Zhaoqi Li, Lillian Ratliff, Kevin~G Jamieson, Lalit Jain, et~al.
\newblock Instance-optimal pac algorithms for contextual bandits.
\newblock \emph{Advances in Neural Information Processing Systems},
  35:\penalty0 37590--37603, 2022.

\bibitem[Mladenov et~al.(2020)Mladenov, Hsu, Jain, Ie, Colby, Mayoraz, Pham,
  Tran, Vendrov, and Boutilier]{mladenov2020demonstrating}
Martin Mladenov, Chih-wei Hsu, Vihan Jain, Eugene Ie, Christopher Colby,
  Nicolas Mayoraz, Hubert Pham, Dustin Tran, Ivan Vendrov, and Craig Boutilier.
\newblock Demonstrating principled uncertainty modeling for recommender
  ecosystems with recsim ng.
\newblock In \emph{Proceedings of the 14th ACM Conference on Recommender
  Systems}, pages 591--593, 2020.

\bibitem[Sakhi et~al.(2023)Sakhi, Alquier, and Chopin]{sakhi2023pac}
Otmane Sakhi, Pierre Alquier, and Nicolas Chopin.
\newblock Pac-bayesian offline contextual bandits with guarantees.
\newblock In \emph{International Conference on Machine Learning}, pages
  29777--29799. PMLR, 2023.

\bibitem[Soare et~al.(2014)Soare, Lazaric, and Munos]{soare2014best}
Marta Soare, Alessandro Lazaric, and R{\'e}mi Munos.
\newblock Best-arm identification in linear bandits.
\newblock \emph{Advances in Neural Information Processing Systems}, 27, 2014.

\bibitem[Strehl et~al.(2009)Strehl, Li, and Littman]{strehl2009reinforcement}
Alexander~L Strehl, Lihong Li, and Michael~L Littman.
\newblock Reinforcement learning in finite mdps: Pac analysis.
\newblock \emph{Journal of Machine Learning Research}, 10\penalty0 (11), 2009.

\bibitem[Wagenmaker and Jamieson(2022)]{wagenmaker2022instance}
Andrew Wagenmaker and Kevin~G Jamieson.
\newblock Instance-dependent near-optimal policy identification in linear mdps
  via online experiment design.
\newblock \emph{Advances in Neural Information Processing Systems},
  35:\penalty0 5968--5981, 2022.

\bibitem[Wainwright(2019)]{wainwright2019high}
Martin~J Wainwright.
\newblock \emph{High-dimensional statistics: A non-asymptotic viewpoint},
  volume~48.
\newblock Cambridge university press, 2019.

\bibitem[Wang(2023)]{wang2023pseudo}
Kaizheng Wang.
\newblock Pseudo-labeling for kernel ridge regression under covariate shift.
\newblock \emph{arXiv preprint arXiv:2302.10160}, 2023.

\bibitem[Xu et~al.(2018)Xu, Honda, and Sugiyama]{xu2018fully}
Liyuan Xu, Junya Honda, and Masashi Sugiyama.
\newblock A fully adaptive algorithm for pure exploration in linear bandits.
\newblock In \emph{International Conference on Artificial Intelligence and
  Statistics}, pages 843--851. PMLR, 2018.

\bibitem[Yang and Tan(2022)]{yang2022minimax}
Junwen Yang and Vincent Tan.
\newblock Minimax optimal fixed-budget best arm identification in linear
  bandits.
\newblock \emph{Advances in Neural Information Processing Systems},
  35:\penalty0 12253--12266, 2022.

\end{thebibliography}
\clearpage
\newpage
\appendix
\onecolumn

\begin{center}
    \textbf{\Huge Appendix}
\end{center}
\vspace{1cm}
\etocdepthtag.toc{mtappendix}
\etocsettagdepth{mtchapter}{none}
\etocsettagdepth{mtappendix}{subsection}
\tableofcontents

\clearpage
\newpage

\section{Proof of Theorem~\ref{theorem; optimal variance design}}\label{section; proof design}
\noindent
We define \(\tilde{p}_2, \dots, \Tilde{p}_K\) as the G-optimal design for \(\{x_2 - x_1, \dots, x_K - x_1\}\), which solves \cref{equation; g-optimal design}.
Also, we define 
\begin{align*}
\bSigma_{\opt, 1} := \sum_{i=2}^K \Tilde{p}_i (x_i - x_1)(x_i - x_1)^\top.
\end{align*}
By the definition of the G-optimal design, we have
\begin{align*}
\max_{i \in [K]} \|x_i - x_1 \|_{ \bSigma_{\opt, 1}^{-1}} \leq \sqrt{d}.
\end{align*}
Recall that we set \(\pb^{\des} = \bigl(\frac{1}{2}, \frac{1}{2}\tilde{p}_2, \dots, \frac{1}{2}\tilde{p}_K\bigr)\) and define \(\bSigma_{\des} := \sum_{i=1}^K p_i^{\des} (x_i - \bar{x})(x_i - \bar{x})^\top\), which is the covariance of the policy \(\pb^{\des}\).
Applying Lemma~\ref{lemma; covariance decomposition}, we obtain
\begin{align*}
\bSigma_{\des} &\succeq p^{\des}_1 \sum_{j \neq 1} p^{\des}_j(x_j - x_1)(x_j - x_1)^\top 
= \frac{1}{2} \sum_{j \neq 1} \frac{1}{2}\tilde{p}_j (x_j - x_1)(x_j - x_1)^\top 
= \frac{1}{4} \bSigma_{\opt, 1}.
\end{align*}
Hence, the covariance \(\bSigma_{\des}\) of our policy satisfies
\begin{align*}
\bSigma_{\des}\succeq \frac{1}{4}\bSigma_{\opt, 1}.
\end{align*} 
We are now ready to prove the theorem.
By the definition of the G-optimal design, \(\bSigma_{\opt, 1}\), we get for all \(i \geq 2\):
\begin{align*}
\norm{x_i - x_1}_{ \bSigma_{\des}^{-1}} \leq 2\norm{x_i - x_1}_{\bSigma_{\opt, 1}^{-1}} \leq 2\sqrt{d},
\end{align*}
and this is the desired result.

Next, we prove the second inequality.
Recall that we define \(\tilde{x}_i := x_i - \bar{x}_{\des}\) and observe  
\begin{align*}
x_i - \bar{x}_{\des} 
= (x_i - x_1) - \sum_{j =1}^K p_j^{\des}(x_j - x_1),
\end{align*}
and hence,
\begin{align*}
\norm{\tilde{x}_i}_{\bSigma_{\des}^{-1}} 
&\leq 2\norm{\tilde{x}_i}_{\bSigma_{\opt, 1}^{-1}} 
\leq 2 \Bigl(\sum_{j=1}^K p_j^{\des} \norm{x_j - x_1}_{\bSigma_{\opt, 1}^{-1}} + \norm{x_i - x_1}_{\bSigma_{\opt,1}^{-1}} \Bigr) \\
&\leq 2 \bigl(\sum_{j=1}^K p_j^{\des} \sqrt{d} + \sqrt{d}\bigr) 
\leq 4\sqrt{d}.
\end{align*}
\hfill \(\BlackBox\)

\begin{lemma}\label{lemma; covariance decomposition}
For any policy \(\pb = (p_1, \dots, p_K) \in \Delta^{(K)}\) and \(\bar{x}_{\pb} = \sum_{i=1}^K p_i x_i\), the following holds:
\begin{align*}
\bSigma_{\pb} := \sum_{i=1}^K p_i (x_i - \bar{x}_\pb)(x_i - \bar{x}_\pb)^\top  
= \sum_{i < j } p_i p_j\bigl(x_i - x_j\bigr)\bigl(x_i - x_j\bigr)^\top.
\end{align*}
\end{lemma}

\begin{proof}
We define \(\out(a,b) := a b^\top\), which is a bilinear function.
Observe that
\begin{align*}
\bSigma_{\pb} 
&:= \sum_{i=1}^K p_i \bigl(x_i - \bar{x}_{\pb}\bigr)\bigl(x_i - \bar{x}_{\pb}\bigr)^\top \\
&= \sum_{i=1}^K p_i \out(x_i, x_i) - \sum_{i=1}^K p_i \out(x_i, \bar{x}_\pb) - \sum_{i=1}^K p_i \out( \bar{x}_\pb,x_i)+ \sum_{i=1}^K p_i \out(\bar{x}_{\pb}, \bar{x}_{\pb})\\
&= \sum_{i=1}^K p_i \out(x_i, x_i) - \out(\bar{x}_{\pb}, \bar{x}_{\pb}) 
\quad (\text{since \(\sum_{i=1}^K p_i =1\)})\\
&= \sum_{i=1}^K p_i \out(x_i, x_i) - \sum_{i=1}^K p_i^2 \out(x_i, x_i) 
- 
\sum_{i<j} p_i p_j (\out(x_i, x_j)+\out(x_j, x_i))\\ 
&= \sum_{i=1}^K (p_i - p_i^2)\out(x_i, x_i)-
\sum_{i < j} p_i p_j(\out(x_i, x_j)+\out(x_j, x_i))\\
&= \sum_{i=1}^K \sum_{j: j \neq i} p_i p_j \out(x_i, x_i)-
\sum_{i < j} p_i p_j (\out(x_i, x_j)+\out(x_j, x_i))\\
&= \sum_{i \neq j} p_i p_j \out(x_i, x_i)-
\sum_{i < j} p_i p_j (\out(x_i, x_j)+\out(x_j, x_i))\\
&= \sum_{i < j} p_i p_j \out(x_i, x_i) + \sum_{i < j} p_i p_j \out(x_j, x_j)- 
\sum_{i < j} p_i p_j (\out(x_i, x_j)+\out(x_j, x_i))\\
&= \sum_{i < j} p_i p_j \out(x_i - x_j, x_i - x_j),
\end{align*}
and this concludes the proof.

\paragraph{Second proof.}\footnote{We thank reviewer $\#$2 for valuable comments about correcting our proofs and for providing this simple and elegant second proof.} 
Let $X$ and $Y$ be independent and identically distributed (i.i.d.) random vectors. The expected outer product of their difference is twice the covariance of $X$:
\begin{align*}
    \mathbb{E}\left[(X-Y)(X-Y)^\top\right] = 2\mathbb{E}\left[XX^\top\right] - 2\mathbb{E}[X]\mathbb{E}[X]^\top = 2\operatorname{Cov}(X).
\end{align*}
Now, consider a discrete random vector $X$ with the distribution $\mathbb{P}(X=x_i) = p_i$ for $i=1, \dots, K$. The left-hand side can be expressed as:
\begin{align*}
    \mathbb{E}\left[(X-Y)(X-Y)^\top\right] = 2\sum_{i < j } p_i p_j\left(x_i - x_j\right)\left(x_i - x_j\right)^\top.
\end{align*}
The covariance of $X$ is given by:
\begin{align*}
  \operatorname{Cov}(X) =   \sum_{i=1}^K p_i (x_i - \bar{x}_\mathbf{p})(x_i - \bar{x}_\mathbf{p})^\top,
\end{align*}
where $\bar{x}_\mathbf{p} = \mathbb{E}[X]$. Equating the two expressions concludes the proof.
\end{proof}

\section{Proof of Theorem~\ref{theorem; dimension optimal error bound}}\label{section; proof error analysis}
\noindent
We define \(\bSigma_{\pb} := \EE[\tilde{x}_{a_s} \tilde{x}_{a_s}^\top] = \sum_{i=1}^K p^{}_i (x_i - \bar{x}_{\pb})(x_i - \bar{x}_{\pb})^\top \).
Abusing notation slightly, we simply write \(\bSigma = \bSigma_{\pb}\) and \(\bar{x} := \bar{x}_{\pb}\) for this proof.
We first prove the upper bound.

\paragraph{Step 0: Preparations.}
We set the ridge regularizer \(\beta_t = \log(t/\delta)\) and the normalized ridge parameter \(\lambda_t = \frac{\beta_t}{t}\).
We set \(\widehat{\Vb}_t := \sum_{i=1}^t \tilde{x}_{a_s}\tilde{x}_{a_s}^\top\) and \(\Vb_t = t \bSigma = \EE[\widehat{\Vb}_t]\). 
Also, set \(\widehat{\bSigma}_t  = \frac{\widehat{\Vb}_t}{t}\). 
For our choice of \(\beta_t = \log(t/\delta)\), by Lemma~\ref{lemma; second moment concentration}, the following holds with probability at least \(1 - \frac{\delta}{10}\) for some absolute constant \(c>0\):
\begin{equation}\label{equation; comparability thm 2}
\frac{1}{c}\bigl(\widehat{\bSigma}_t + \lambda_t \Ib_d \bigr)\preceq \bSigma + \lambda_t \Ib_d \preceq c\bigl(\widehat{\bSigma}_t + \lambda_t \Ib_d \bigr), 
\quad    
\frac{1}{c}\bigl(\widehat{\Vb}_t + \beta_t \Ib_d \bigr)\preceq \Vb_t + \beta_t \Ib_d \preceq c\bigl(\widehat{\Vb}_t + \beta_t \Ib_d \bigr).
\end{equation}
Recall that we define \(\tilde{x}_{a_s} = x_{a_s} - \bar{x}\), and then \(\EE[\tilde{x}_{a_s} \mid \cH_{s-1}] = 0\).
Also, we have \(\EE[\tilde{x}_{a_s}\tilde{x}_{a_s}^\top] = \bSigma\). 

We prove the upper bound of the theorem in four steps; afterward, we show its optimality.

\paragraph{Step 1: Error Decomposition.}
Using the definition of \(\hat{\theta}_t\) from \cref{equation; estimator}, decompose the estimation error \(\hat{\theta}_t - \theta^\star\) as
\begin{align*}
\hat{\theta}_t -\theta^\star 
&=  (\widehat{\Vb}_t + \beta_t \Ib_d)^{-1} \sum_{s=1}^t \tilde{x}_{a_s} (x_{a_s}^\top \theta^\star + \nu_s + \eta_s) 
- (\widehat{\Vb}_t + \beta_t  \Ib_d)^{-1} \bigl(\widehat{\Vb}_t + \beta_t  \Ib_d\bigr)\theta^\star\\
&=(\widehat{\Vb}_t + \beta_t  \Ib_d)^{-1} \sum_{s=1}^t \tilde{x}_{a_s}\bigl(\tilde{x}_{a_s}^\top \theta^\star+\bar{x}^\top \theta^\star + \nu_s + \eta_s\bigr) \\
&\quad - (\widehat{\Vb}_t + \beta_t  \Ib_d)^{-1} \bigl(\sum_{s=1}^t  \tilde{x}_{a_s}\tilde{x}_{a_s}^\top \theta^\star \bigr) - (\widehat{\Vb}_t + \beta_t  \Ib_d)^{-1} \beta_t \theta^\star\\
&= (\widehat{\Vb}_t + \beta_t  \Ib_d)^{-1} \sum_{s=1}^t \tilde{x}_{a_s}\bigl(\bar{x}^\top \theta^\star + \nu_s + \eta_s\bigr) 
- \beta_t (\widehat{\Vb}_t + \beta_t  \Ib_d)^{-1}\theta^\star \\
&=  \underbrace{(\widehat{\Vb}_t + \beta_t  \Ib_d)^{-1} \sum_{s=1}^t \tilde{x}_{a_s}\underbrace{\bigl(\bar{x}^\top \theta^\star + \nu_s\bigr)}_{:=q_s}}_{:=\Acr} 
+ \underbrace{(\widehat{\Vb}_t + \beta_t  \Ib_d)^{-1} \sum_{s=1}^t \tilde{x}_{a_s}\eta_s}_{:=\Bcr} 
- \underbrace{\beta_t (\widehat{\Vb}_t + \beta_t  \Ib_d)^{-1}\theta^\star}_{:=\Ccr} \\
&:= \Acr + \Bcr + \Ccr.
\end{align*}

\paragraph*{Step 2: Bounding \(\Ccr\).}
For any \(z \in \RR^d\), we can bound \(z^\top \Ccr\) with probability at least \(1-\frac{\delta}{10}\) using \cref{equation; comparability thm 2}:
\begin{align*}
|z^\top \Ccr | 
&\leq \bigl|\sqrt{\beta_t} z^\top (\widehat{\Vb}_t + \beta_t  \Ib_d)^{-\frac{1}{2}} (\widehat{\Vb}_t + \beta_t  \Ib_d)^{-\frac{1}{2}} \sqrt{\beta_t}\theta^\star\bigr| \\
&\leq \sqrt{\beta_t}\bigl\| z^\top (\widehat{\Vb}_t + \beta_t  \Ib_d)^{-\frac{1}{2}} \bigr\|_2  
\bigl\|\sqrt{\beta_t}(\widehat{\Vb}_t + \beta_t  \Ib_d)^{-\frac{1}{2}} \theta^\star \bigr\|_2\\
&\leq \sqrt{\beta_t}\frac{1}{\sqrt{t}}\sqrt{ z^\top \bigl(\widehat{\bSigma}_t +\lambda_t \Ib_d \bigr)^{-1} z } \times \sqrt{\beta_t}\bigl\| (\widehat{\Vb}_t + \beta_t  \Ib_d)^{-\frac{1}{2}} \theta^\star \bigr\|_2 \\
&\lesssim \sqrt{\beta_t}\frac{1}{\sqrt{t}}\sqrt{ z^\top \bigl(\bSigma +\lambda_t \Ib_d \bigr)^{-1} z } \times \sqrt{\beta_t}\bigl\| (\widehat{\Vb}_t + \beta_t  \Ib_d)^{-\frac{1}{2}} \theta^\star \bigr\|_2  \quad \text{(by \cref{equation; comparability thm 2})}\\
&\lesssim \frac{\sqrt{\beta_tL}}{\sqrt{t}}\|\theta^\star \|_2 
\lesssim \frac{\sqrt{\beta_tL}}{\sqrt{t}} 
= \frac{\sqrt{\log(t/\delta)L}}{\sqrt{t}}.
\end{align*}

\paragraph*{Step 3: Bounding \(\Bcr\).}
Next, we bound \(z^\top \Bcr\). Recall
\begin{align*}
z^\top \Bcr 
= \sum_{s=1}^t z^\top \bigl(\widehat{\Vb}_t + \beta_t  \Ib_d\bigr)^{-1}\tilde{x}_{a_s}\eta_s.
\end{align*}
Given $\{\tilde{x}_{a_1}, \dots \tilde{x}_{a_t}\}$, the random variable \( z^\top \bigl(\widehat{\Vb}_t + \beta_t  \Ib_d\bigr)^{-1}\tilde{x}_{a_s}\eta_s \) is mean-zero and sub-Gaussian with proxy \(\alpha_s := |z^\top \bigl(\widehat{\Vb}_t + \beta_t  \Ib_d\bigr)^{-1}\tilde{x}_{a_s}|\) (since noise $\eta_s$ is sampled independently).
See that
\begin{align*}
\sum_{s=1}^t \alpha_s^2 
&= \sum_{s=1}^t z^\top \bigl(\widehat{\Vb}_t + \beta_t  \Ib_d\bigr)^{-1}\tilde{x}_{a_s}\tilde{x}_{a_s}^\top \bigl(\widehat{\Vb}_t + \beta_t  \Ib_d\bigr)^{-1} z \\
&= z^\top \bigl(\widehat{\Vb}_t + \beta_t  \Ib_d\bigr)^{-1}\widehat{\Vb}_t\bigl(\widehat{\Vb}_t + \beta_t  \Ib_d\bigr)^{-1} z \\
&\leq z^\top \bigl(\widehat{\Vb}_t + \beta_t  \Ib_d\bigr)^{-1} z.
\end{align*}
Applying Bernstein's inequality for sub-Gaussian random variables \citep{boucheron2013concentration}, with probability at least \(1-\frac{\delta}{10}\),
\begin{align*}
|z^\top \Bcr| 
\lesssim \sqrt{\frac{z^\top \bigl(\widehat{\bSigma}_t + \lambda_t  \Ib_d\bigr)^{-1} z}{t}\log\bigl(\frac{1}{\delta}\bigr)}.
\end{align*}
Since \cref{equation; comparability thm 2} holds with probability at least \(1-\frac{\delta}{10}\), we have with probability at least \(1-\frac{2\delta}{10}\):
\begin{align*}
|z^\top \Bcr| 
\lesssim \sqrt{\frac{z^\top \bigl(\bSigma + \lambda_t  \Ib_d\bigr)^{-1} z}{t}\log\bigl(\frac{1}{\delta}\bigr)}
\lesssim \frac{\sqrt{L \log\bigl(\frac{t}{\delta}\bigr)}}{\sqrt{t}}.
\end{align*}

\paragraph{Step 3: Bounding \(\Acr\).}
Lastly, we bound the most challenging term, \(\Acr\).
We define \(q_s := \bar{x}^\top \theta^\star + \nu_s\).
Under the boundedness Assumption~\ref{assumption; boundedness}, \(|q_s| \leq 2\).
Observe that
\begin{align*}
\Acr 
&:= \frac{1}{t}\bigl(\widehat{\bSigma}_t + \lambda_t  \Ib_d\bigr)^{-1} \sum_{s=1}^t (\bSigma + \lambda_t  \Ib_d)^{\frac{1}{2}} (\bSigma + \lambda_t  \Ib_d)^{-\frac{1}{2}} \tilde{x}_{a_s}q_s \\
&= \frac{1}{t}\bigl(\widehat{\bSigma}_t + \lambda_t  \Ib_d\bigr)^{-1} (\bSigma + \lambda_t  \Ib_d)^{\frac{1}{2}} \sum_{s=1}^t \underbrace{(\bSigma + \lambda_t \Ib_d)^{-\frac{1}{2}}\tilde{x}_{a_s}}_{:= e_s} q_s \\
&:= \frac{1}{t}\bigl(\widehat{\bSigma}_t + \lambda_t  \Ib_d\bigr)^{-1} (\bSigma + \lambda_t  \Ib_d)^{\frac{1}{2}} \sum_{s=1}^t e_s q_s.
\end{align*}
Here, \(e_s := (\bSigma+\lambda_t  \Ib_d)^{-\frac{1}{2}}\tilde{x}_{a_s} \in \RR^d\) is mean zero given $\cH_{s-1}$, and its variance satisfies
\begin{align*}
    \operatorname{var}(e_s \mid \cH_{s-1}) 
&= (\bSigma+\lambda_t  \Ib_d)^{-\frac{1}{2}}\EE\bigl[\tilde{x}_{a_s}\tilde{x}_{a_s}^\top \mid \cH_{s-1}\bigr](\bSigma+\lambda_t  \Ib_d)^{-\frac{1}{2}} \\
&= (\bSigma+\lambda_t  \Ib_d)^{-\frac{1}{2}}  \bSigma(\bSigma+\lambda_t  \Ib_d)^{-\frac{1}{2}} \\
&\preceq \Ib_d.
\end{align*}
Also, \(\|e_s\|^2_2 \leq M\) by the given condition.
Then, for any \(z \in \Xcal - x_1\), we have 
\begin{align*}
z^\top \Acr 
&= \frac{1}{t}z^\top \bigl(\widehat{\bSigma}_t + \lambda_t  \Ib_d\bigr)^{-1} (\bSigma + \lambda_t  \Ib_d)^{\frac{1}{2}} \sum_{s=1}^t e_s q_s \\
&= \frac{1}{t}z^\top \bigl(\bSigma + \lambda_t  \Ib_d\bigr)^{-1} (\bSigma + \lambda_t  \Ib_d)^{\frac{1}{2}} \sum_{s=1}^t e_s q_s \\
&\quad + \frac{1}{t}z^\top \bigl(\widehat{\bSigma}_t + \lambda_t  \Ib_d\bigr)^{-1} \bigl(\bSigma + \lambda_t  \Ib_d - \widehat{\bSigma}_t - \lambda_t  \Ib_d\bigr)\bigl(\bSigma + \lambda_t  \Ib_d\bigr)^{-1} (\bSigma + \lambda_t  \Ib_d)^{\frac{1}{2}} \sum_{s=1}^t e_s q_s \\
&= \frac{1}{t}z^\top \bigl(\bSigma + \lambda_t  \Ib_d\bigr)^{-\frac{1}{2}} \sum_{s=1}^t e_s q_s 
+ \frac{1}{t}z^\top \bigl(\widehat{\bSigma}_t + \lambda_t  \Ib_d\bigr)^{-1} \bigl(\bSigma - \widehat{\bSigma}_t\bigr)\bigl(\bSigma + \lambda_t  \Ib_d\bigr)^{-\frac{1}{2}} \sum_{s=1}^t e_s q_s \\
&:= I + II,
\end{align*}
where we used the identity \(\Ab^{-1} - \Bb^{-1} = \Ab^{-1}(\Bb-\Ab)\Bb^{-1}\) in the second equality.

\vspace{0.4cm}
\noindent \underline{[3-1] Bounding Term \(I\).}\;\;
We first bound term \(I\).
See that
\begin{align*}
|I| 
&\leq \frac{1}{t}\Bigl|z^\top (\bSigma + \lambda_t  \Ib_d)^{-\frac{1}{2}} \sum_{s=1}^t e_s q_s\Bigr| 
= \frac{1}{t}\bigl|\sum_{s=1}^t q_sz^\top (\bSigma + \lambda_t  \Ib_d)^{-\frac{1}{2}}  e_s\bigr|.
\end{align*}
Since \(|q_s|\leq 2\) and \(\EE[e_s e_s^\top \mid \cH_{s-1}] \preceq \Ib_d\), we get
\begin{align*}
\operatorname{var}\bigl(q_sz^\top (\bSigma + \lambda_t  \Ib_d)^{-\frac{1}{2}}  e_s \big| \cH_{s-1}\bigr) 
&= q_s^2 \EE\bigl[z^\top (\bSigma + \lambda_t  \Ib_d)^{-\frac{1}{2}}  e_s e_s^\top(\bSigma + \lambda_t  \Ib_d)^{-\frac{1}{2}} z \big| \cH_{s-1}\bigr] \\
&\leq 4 z^\top (\bSigma + \lambda_t  \Ib_d)^{-1} z 
\lesssim L.
\end{align*}
Also,
\begin{align*}
\bigl|q_sz^\top (\bSigma + \lambda_t  \Ib_d)^{-\frac{1}{2}}  e_s\bigr| 
\lesssim \bigl\|z^\top (\bSigma + \lambda_t  \Ib_d)^{-\frac{1}{2}}\bigr\|_2 \|e_s\|_2 
\lesssim \sqrt{L}\sqrt{M}
\lesssim \sqrt{LM}.
\end{align*}
Since $\EE[e_s\mid \cH_{s-1}]=0$ and $q_s$ is measurable w.r.t. $\cH_{s-1}$, by applying Bernstein's inequality for sum of martingale differences (see \citealt{fan2019bernstein}), we get with probability at least \(1-\frac{\delta}{10}\):
\begin{align*}
|I| 
\lesssim \frac{1}{t}\sqrt{Lt \log\bigl(\frac{10}{\delta}\bigr)} + \frac{1}{t}\sqrt{LM}\log\bigl(\frac{10}{\delta}\bigr)
\lesssim \sqrt{\frac{L}{t}\log\bigl(\frac{1}{\delta}\bigr)} 
+\frac{\sqrt{LM}}{t}\log\bigl(\frac{1}{\delta}\bigr).
\end{align*}

\vspace{0.5cm}

\noindent \underline{[3-2] Bounding Term \(II\).}\;\;
This term is a higher order term, such as $\tilde{\cO}(\frac{1}{t})$.
Observe that
\begin{align*}
|II|
&= \Bigl|\frac{1}{t}z^\top \bigl(\widehat{\bSigma}_t + \lambda_t  \Ib_d\bigr)^{-1}\bigl(\bSigma - \widehat{\bSigma}_t\bigr) (\bSigma + \lambda_t  \Ib_d)^{-\frac{1}{2}} \sum_{s=1}^t e_sq_s\Bigr| \\
&= \Bigl|\frac{1}{t}z^\top \underbrace{\bigl(\bSigma + \lambda_t  \Ib_d\bigr)^{-\frac{1}{2}} (\bSigma + \lambda_t  \Ib_d\bigr)^{\frac{1}{2}}}_{\blue{}} \bigl(\widehat{\bSigma}_t + \lambda_t  \Ib_d\bigr)^{-1} \\
&\quad \times
\underbrace{(\bSigma + \lambda_t  \Ib_d)^{\frac{1}{2}}(\bSigma + \lambda_t  \Ib_d)^{-\frac{1}{2}}}_{\purple{}} \bigl(\bSigma - \widehat{\bSigma}_t\bigr) (\bSigma + \lambda_t  \Ib_d)^{-\frac{1}{2}} \sum_{s=1}^t e_sq_s\Bigr| \\
&\leq \frac{1}{t}\bigl\|z^\top (\bSigma + \lambda_t  \Ib_d)^{-\frac{1}{2}}\bigr\|_2  
\bigl\|(\bSigma + \lambda_t  \Ib_d)^{\frac{1}{2}} (\widehat{\bSigma}_t + \lambda_t  \Ib_d\bigr)^{-1} (\bSigma + \lambda_t  \Ib_d)^{\frac{1}{2}}\bigr\|_{\op} \\
&\quad  \times \bigl\|(\bSigma + \lambda_t  \Ib_d)^{-\frac{1}{2}}\bigl(\bSigma - \widehat{\bSigma}_t\bigr) (\bSigma + \lambda_t  \Ib_d)^{-\frac{1}{2}}\bigr\|_{\op} \bigl\|\sum_{s=1}^t e_sq_s \bigr\|_2 \\
&\leq \frac{1}{t}\sqrt{L}\bigl\| (\bSigma + \lambda_t  \Ib_d)^{\frac{1}{2}} (\widehat{\bSigma}_t + \lambda_t  \Ib_d)^{-1} (\bSigma + \lambda_t  \Ib_d)^{\frac{1}{2}}\bigr\|_{\op} \\
&\quad  \times \bigl\|(\bSigma + \lambda_t  \Ib_d)^{-\frac{1}{2}}\bigl(\bSigma - \widehat{\bSigma}_t\bigr) (\bSigma + \lambda_t  \Ib_d)^{-\frac{1}{2}}\bigr\|_{\op} \bigl\|\sum_{s=1}^t e_sq_s \bigr\|_2.
\end{align*}
First, with probability at least \(1-\frac{\delta}{10}\), \cref{equation; comparability thm 2} gives
\begin{align*}
\bigl\| (\bSigma + \lambda_t  \Ib_d)^{\frac{1}{2}} (\widehat{\bSigma}_t + \lambda_t  \Ib_d)^{-1} (\bSigma + \lambda_t  \Ib_d)^{\frac{1}{2}}\bigr\|_{\op} 
\leq 2.
\end{align*}
Next, by applying the dimension-free martingale difference bound (Lemma 10 of \citet{kim2021doubly}), we aim to bound \(\|\sum_{s=1}^t e_sq_s\|_2\).
Since \(\|q_se_s\|_2 \leq 2\sqrt{M}\) and \(\EE[e_s \mid \cH_{s-1}] = 0\), we can bound it with probability at least \(1-\frac{\delta}{10}\):
\begin{align*}
\Bigl\|\sum_{s=1}^t e_s q_s\Bigr\|_2 
\lesssim \sqrt{M}\sqrt{t \log\bigl(\frac{10}{\delta}\bigr)} 
\lesssim \sqrt{tM \log\bigl(\frac{1}{\delta}\bigr)}.
\end{align*}
Lastly, to bound 
\(\bigl\|(\bSigma + \lambda_t  \Ib_d)^{-\frac{1}{2}}\bigl(\bSigma - \widehat{\bSigma}_t\bigr)(\bSigma + \lambda_t  \Ib_d)^{-\frac{1}{2}}\bigr\|_{\op}\), we use the result of Lemma~\ref{lemma; covariance concentration} and obtain with probability at least \(1-\frac{\delta}{10}\):
\begin{align*}
\bigl\|(\bSigma + \lambda_t  \Ib_d)^{-\frac{1}{2}}\bigl(\bSigma - \widehat{\bSigma}_t\bigr)(\bSigma + \lambda_t  \Ib_d)^{-\frac{1}{2}}\bigr\|_{\op}
\lesssim \sqrt{\frac{M\log\bigl(\frac{10d}{\delta}\bigr)}{t}}
\lesssim \sqrt{\frac{M\log\bigl(\frac{d}{\delta}\bigr)}{t}}.
\end{align*}
Combining the above, with probability at least \(1-\frac{3\delta}{10}\):
\begin{align*}
|II| 
&\lesssim \frac{\sqrt{L}}{t}\bigl\| (\bSigma + \lambda_t  \Ib_d)^{\frac{1}{2}} (\widehat{\bSigma}_t + \lambda_t  \Ib_d)^{-1} (\bSigma + \lambda_t  \Ib_d)^{\frac{1}{2}}\bigr\|_{\op} \\
&\quad \times \bigl\|(\bSigma + \lambda_t  \Ib_d)^{-\frac{1}{2}}\bigl(\bSigma - \widehat{\bSigma}_t\bigr)(\bSigma + \lambda_t  \Ib_d)^{-\frac{1}{2}}\bigr\|_{\op} \Bigl\|\sum_{i=1}^t e_ia_i\Bigr\|_2  \\
&\lesssim \frac{\sqrt{L}}{t}\frac{\sqrt{M\log\bigl(\frac{d}{\delta}\bigr)}}{\sqrt{t}} \sqrt{t}\sqrt{M \log\bigl(\frac{1}{\delta}\bigr)} \\
&\lesssim \frac{\sqrt{L}M\log\bigl(\frac{d}{\delta}\bigr)}{t}.
\end{align*}

\paragraph{Step 4: Final Upper Bound.}
Combining the established bounds, with probability at least \(1-\delta\), we have 
\begin{align*}
|z^\top(\hat{\theta}_t-\theta^\star)| 
\lesssim \frac{\sqrt{\log\bigl(\frac{t}{\delta}\bigr)L}}{\sqrt{t}} 
+\frac{\sqrt{L}M}{t}\log\bigl(\frac{d}{\delta}\bigr).
\end{align*}

\paragraph{Step 5: Proof of Optimality.}
Consider the case when the shift is given by \(\nu_s = -\bar{x}_\pb \theta^\star\) for all \(1 \le s \le t\).
Then, our problem has the reward structure 
\[
r_{s} = (x_{a_s}-\bar{x}_\pb)^\top \theta^\star + \eta_s,
\]
which corresponds to a linear model with independent variables (covariates) \(\{x_{a_s} - \bar{x}_\pb\}_{s=1}^t\) and responses \(\{r_s\}_{s=1}^t\).
However, our problem is more challenging since we do not know the exact shift (so we do not know this is actually a linear model), nor do we know \(\{x_{a_s} - \bar{x}_\pb\}\) are the covariates of the linear model. 

In this case, the second-moment matrix of the linear model's covariates is
\[
\EE\bigl[(x_{a_s} - \bar{x}_\pb)(x_{a_s} - \bar{x}_\pb)^\top\bigr] 
= \bSigma_\pb.
\]
Thus, the known lower bound for the estimation error at \(z\) for this linear model is \(\frac{\sqrt{L}}{\sqrt{t}}\).
This result is widely known, but we include a proof reference for completeness.
We can interpret this as a transfer learning problem in a linear model, where the source distribution has a second-moment matrix \(\bSigma_\pb\), while the target distribution is a Dirac distribution at \(z\).
By applying Theorem 3.2 of \citet{ge2023maximum}, we obtain the desired result directly.

\hfill \(\BlackBox\)

\noindent
We now present the lemma and its proof used in the proof of Theorem~\ref{theorem; dimension optimal error bound}.

\begin{lemma}\label{lemma; covariance concentration}
For any \(\lambda > 0\),
\begin{align*}
\bigl\| (\bSigma + \lambda \Ib)^{-\frac{1}{2}} \bigl(\widehat{\bSigma}_t -\bSigma\bigr)   (\bSigma + \lambda \Ib)^{-\frac{1}{2}}\bigr\|_{\op} 
\lesssim \sqrt{\frac{M\log\bigl(\frac{d}{\delta}\bigr)}{t}}
\end{align*}
holds with probability at least \(1-\delta\).
\end{lemma}

\begin{proof}
We apply Lemma~\ref{lemma; operator covariance concentration} to \(\bm x_s = (\bSigma + \lambda \Ib)^{-\frac{1}{2}} \tilde{x}_{a_s}\).
Then \(\|\bm x_s\|_2 \leq \sqrt{M}\), so we set \(S \leftarrow \sqrt{M}\).
Next, we set \(v = 1\) in the lemma because 
\(\|\EE[\bm x_s \bm x_s^\top]\|_{\op} =  \|(\bSigma + \lambda \Ib)^{-\frac{1}{2}} \bSigma (\bSigma + \lambda \Ib)^{-\frac{1}{2}}\|_{\op} \leq 1.\)
Finally, we set \(r = d\) because
\(\tr(\EE[\bm x_s \bm x_s^\top]) 
= \tr\bigl((\bSigma + \lambda \Ib)^{-\frac{1}{2}} \bSigma (\bSigma + \lambda \Ib)^{-\frac{1}{2}}\bigr) 
\leq \tr\bigl((\bSigma + \lambda \Ib)^{-\frac{1}{2}} (\bSigma + \lambda \Ib)(\bSigma + \lambda \Ib)^{-\frac{1}{2}}\bigr) 
= d.\)
\end{proof}

\section{Proof of Corollary~\ref{corollary; dimension optimal sample complexity}}\label{section; proof sample complexity}
\noindent
Assume that we use the greedy policy with the estimator \(\hat{\theta}_t\), i.e., \(a_t = \arg \max_{i \in [K]} x_i^\top \hat{\theta}_t\) and it is equivalent to \(a_t = \arg \max_{i \in [K]} \bigl(x_i -x_1 \bigr)^\top \hat{\theta}_t\).
Suppose \(\max_{i \in [K]}\lvert (x_i -x_1)^\top (\hat{\theta}_t-\theta^\star)\rvert \leq \frac{\varepsilon}{2}\) holds.
Then
\begin{align*}
& \bigl(x_{a_t} -x_1\bigr)^\top \hat{\theta}_t \geq \bigl(x_{a^\star} -x_1\bigr)^\top \hat{\theta}_t,\\
& \bigl(x_{a_t} -x_1\bigr)^\top \theta^\star \leq \bigl(x_{a^\star} -x_1\bigr)^\top \theta^\star,
\end{align*}
holds and it leads
\begin{align*}
\bigl(x_{a^\star} -x_{a_t}\bigr)^\top \theta^\star \leq \varepsilon.
\end{align*}
This implies that our policy is \(\varepsilon\)-PAC.

Hence, to achieve the \((\varepsilon, \delta)\)-PAC property, by Theorem~\ref{theorem; dimension optimal error bound}, it suffices to ensure
\[
C_1\Bigl(\frac{\sqrt{ d\log\bigl(\frac{Kt}{\delta}\bigr)}}{\sqrt{t}} 
+ \frac{d^{\frac{3}{2}} \log\bigl(\frac{dK}{\delta}\bigr)}{t}\Bigr)
< \frac{\varepsilon}{2}.
\]
The term $K$ appears to guarantee concentrations over all $\cX- x_1$.
If 
\begin{align}
&C_1 \frac{\sqrt{  d \log\bigl(\frac{Kt}{\delta}\bigr)}}{\sqrt{t}} < \frac{\varepsilon}{4}, \label{equation; Cor 5 1}\\
&C_1 \frac{d\log\bigl(\frac{dK}{\delta}\bigr)}{t}  < \frac{\varepsilon}{4}, \label{equation; Cor 5 2}
\end{align}
then the above relation also holds. 
The first inequality \eqref{equation; Cor 5 1} suffices if
\[
C_1\frac{\sqrt{  d\log t}}{\sqrt{t}} < \frac{\varepsilon}{8}
\quad\text{and}\quad 
C_1\frac{\sqrt{d \log\bigl(\frac{K}{\delta}\bigr)}}{\sqrt{t}} < \frac{\varepsilon}{8}.
\]
A sufficient condition is
\begin{align*}
&\frac{t}{\log t} > c\frac{d}{\varepsilon^2} \quad 
\text{and} \quad
t >  c\frac{d\log\bigl(\frac{K}{\delta}\bigr)}{\varepsilon^2},\\
&\Longleftarrow \quad
t \gtrsim \frac{d\log\bigl(\frac{d}{\varepsilon}\bigr)}{\varepsilon^2} 
+ \frac{d\log\bigl(\frac{K}{\delta}\bigr)}{\varepsilon^2}
 \\
& \Longleftarrow \quad
t \gtrsim \frac{d\log\bigl(\frac{dK}{\varepsilon \delta}\bigr)}{\varepsilon^2}.
\end{align*}
The second inequality \eqref{equation; Cor 5 2} is equivalent to 
\[
t \gtrsim \frac{d^{\frac{3}{2}}\log\bigl(\frac{dK}{\delta}\bigr)}{\varepsilon}.
\]
Hence, if
\[
t \gtrsim \frac{d\log\bigl(\frac{dK}{\varepsilon \delta}\bigr)}{\varepsilon^2} 
+ \frac{d^{\frac{3}{2}}\log\bigl(\frac{dK}{\delta}\bigr)}{\varepsilon},
\]
we obtain the desired \(\varepsilon\)-PAC property with probability at least \(1-\delta\).

\section{Proofs for Section~\ref{section; low regret algorithm}}\label{section; proof main algorithm}
\noindent
We define the end-time of phase \(\ell\) as \(t_{\ell}\).
Recall that we denote the estimator computed at the end of phase \(\ell\) by \(\hat{\theta}_{(\ell)}\).
Define \(L_\star := \lceil \log_2(1/\Delta_{\star}) \rceil +1\).
We set \(L_T\) to be the last phase until time \(T\).
Note that \(L_T, L_\star\) are deterministic values.
Recall that $\cA_\ell(1)$ is defined as the arm of $\cA_\ell$ with smallest index. 

\subsection{Framework for Regret Analysis}

\begin{lemma}[Good event]\label{lemma; key 1}
With probability at least \(1-\delta\),
\[
\bigl|x^\top\bigl(\hat{\theta}_{(\ell)}-\theta^\star\bigr)\bigr| \leq \frac{\varepsilon_{\ell}}{2}
\]
for any \(x \in \cA_\ell -\cA_\ell(1)\) and any phase \(1 \leq \ell \leq L_T\).
We call this event \(\Ecr_T\).
\end{lemma}

\begin{proof}
By Corollary~\ref{corollary; dimension optimal sample complexity}, for any phase \(\ell\),
\[
\PP\Bigl[\exists \; x \in \cA_\ell -\cA_\ell(1)  \, \text{ s.t. } \, \bigl\lvert x^\top\bigl(\hat{\theta}_{(\ell)}-\theta^\star\bigr)\bigr\rvert 
>\frac{ \varepsilon_{\ell}}{2}\Bigr] 
\leq \frac{\delta}{\ell(\ell+1)}.
\]
By taking the union bound over all \(\ell\), we obtain the desired result.
\end{proof}

\begin{lemma}\label{lemma; key 2}
Under the event \(\Ecr_T\), for each phase \(\ell\), the action set \(\cA_\ell\) contains \(a^\star\).
\end{lemma}

\begin{proof}
We prove this by contradiction. 
Suppose \(a^\star\) is eliminated after some phase \(\ell\).
Then there exists some arm \(a'\) such that 
\[
x_{a'}^\top \hat{\theta}_{(\ell)} > x_{a^\star}^\top \hat{\theta}_{(\ell)} + \varepsilon_\ell \; \Leftrightarrow \;
(x_{a'} -x_{\cA_\ell(1)})^\top \hat{\theta}_{(\ell)} > (x_{a^\star} -x_{\cA_\ell(1)})^\top  \hat{\theta}_{(\ell)} + \varepsilon_\ell.
\]
However, under the event \(\Ecr_T\), by Lemma~\ref{lemma; key 1}, we have
\[
\bigl\lvert (x_{a'} -x_{\cA_\ell(1)})^\top  \hat{\theta}_{(\ell)} - (x_{a'} -x_{\cA_\ell(1)})^\top  \theta^\star\bigr\rvert \leq \frac{\varepsilon_\ell}{2}, 
\;\,
\bigl\lvert (x_{a^\star} -x_{\cA_\ell(1)})^\top  \hat{\theta}_{(\ell)} - (x_{a^\star} -x_{\cA_\ell(1)})^\top \theta^\star\bigr\rvert \leq \frac{\varepsilon_\ell}{2}
\]
and 
\[(x_{a^\star} -x_{\cA_\ell(1)})^\top  \theta^\star > (x_{a^\prime} -x_{\cA_\ell(1)})^\top  \theta^\star\]
which leads a contradiction.
\end{proof} 

\begin{lemma}\label{lemma; key 3}
    Under the event $\Ecr_T$, any arm $a$ in $\cA_\ell$ satisfies 
    \begin{align*}
     x_a^\top \theta^\star >    x_{a^\star}^\top \theta^\star -4\varepsilon_{\ell}.
    \end{align*}
\end{lemma}

\begin{proof}
Since $a^\star \in \cA_{\ell-1}$ and arm $a$ is not eliminated, we have 
\begin{align*}
      &\quad \; x_a^\top \hat \theta_{(\ell-1)} >   x_{a^\star}^\top \hat \theta_{(\ell-1)} -\varepsilon_{\ell-1} \\
      &\Leftrightarrow (x_a -x_{\cA_{\ell-1}(1)})^\top  \hat \theta_{(\ell-1)} >   (x_{a^\star} -x_{\cA_{\ell-1}(1)})^\top \hat \theta_{(\ell-1)} -\varepsilon_{\ell-1}.
\end{align*}
Under the event $\Ecr_T$, by Lemma~\ref{lemma; key 1}, we get 
\begin{align*}
 &\bigl\lvert (x_{a^\prime} -x_{\cA_{\ell-1}(1)})^\top  \hat{\theta}_{(\ell-1)} - (x_{a^\prime} -x_{\cA_{\ell-1}(1)})^\top  \theta^\star\bigr\rvert \leq \frac{\varepsilon_{\ell-1}}{2}  \\
& \bigl\lvert (x_{a^\star} -x_{\cA_{\ell-1}(1)})^\top  \hat{\theta}_{(\ell-1)} - (x_{a^\star} -x_{\cA_{\ell-1}(1)})^\top  \theta^\star\bigr\rvert \leq \frac{\varepsilon_{\ell-1}}{2},
\end{align*}
and it leads 
\begin{align*}
     (x_a -x_{\cA_{\ell-1}(1)})^\top  \theta^\star >  (x_{a^\star} -x_{\cA_{\ell-1}(1)})^\top  \theta^\star -2\varepsilon_{\ell-1}.
\end{align*}
which is equivalent to 
\begin{align*}
     x_a^\top \theta^\star >   x_{a^\star}^\top \theta^\star -2\varepsilon_{\ell-1}.
\end{align*}
\end{proof}

\begin{lemma}[Regret formula without gap]\label{lemma; regret formula without gap}
Under the event \(\Ecr_T\), the regret is bounded by
\[
\Regret(T) \lesssim d2^{L_T}\bigl(\log \bigl(\frac{dKL_T}{\delta}\bigr)+ L_T\bigr) 
+ d^{\frac{3}{2}}\log\bigl(\frac{dKL_T}{\delta}\bigr) L_T.
\]
\end{lemma}

\begin{proof}
Under the event \(\Ecr_T\), by Lemma~\ref{lemma; key 3}, we have
\[
\Regret(T) \lesssim \sum_{\ell=1}^{L_T} n_{\ell}\varepsilon_{\ell}
\lesssim\sum_{\ell=1}^{L_T}\Bigl(\frac{d\log\bigl(\frac{dK\ell}{\delta \varepsilon_\ell}\bigr)}{\varepsilon_\ell^2} +\frac{d^{\frac{3}{2}}}{\varepsilon_\ell}\log\bigl(\frac{dK \ell}{\delta}\bigr)\Bigr)\varepsilon_\ell.
\]
Hence,
\begin{align*}
\Regret(T) 
&\lesssim\sum_{\ell=1}^{L_T} \frac{d\log\bigl(\frac{dKL_T}{\delta \varepsilon_\ell}\bigr)}{\varepsilon_\ell} 
+ d^{\frac{3}{2}}\log\bigl(\frac{dK \ell}{\delta}\bigr)\\
&\lesssim\sum_{\ell=1}^{L_T} \bigl(d\log\bigl(\frac{dKL_T}{\delta}\bigr) + d\log\frac{1}{\varepsilon_\ell}\bigr)\frac{1}{\varepsilon_\ell} 
+ d^{\frac{3}{2}}\log\bigl(\frac{dKL_T}{\delta}\bigr).
\end{align*}
Noting that \(\log(\frac{1}{\varepsilon_\ell}) \leq \ell\), the above is
\begin{align*}
\Regret(T)
&\lesssim\sum_{\ell=1}^{L_T} \frac{d\log\bigl(\frac{dKL_T}{\delta}\bigr) + d\ell}{\varepsilon_\ell}
\,\, +   L_Td^{\frac{3}{2}}\log\bigl(\frac{dKL_T}{\delta}\bigr)\\
&\lesssim d2^{L_T}\Bigl(\log\bigl(\frac{dKL_T}{\delta}\bigr) + L_T\Bigr) 
+ d^{\frac{3}{2}}\log\bigl(\frac{dK L_T}{\delta}\bigr) L_T.  
\end{align*}
\end{proof}

\begin{lemma}\label{lemma; upper bound phase}
The total number of phases \(L_T\) is bounded by 
\[
2^{L_T} \leq c\sqrt{\frac{T}{d\log\bigl(\frac{dK}{\delta}\bigr)}}.
\]
\end{lemma}

\begin{proof}
Observe that 
\[
T 
\gtrsim c\sum_{\ell=1}^{L_T-1} d\log\Bigl(\frac{dK}{\delta}\Bigr)4^\ell
\gtrsim cd\log\Bigl(\frac{dK}{\delta}\Bigr)4^{L_T}.
\]
Hence 
\[
4^{L_T}
\lesssim\frac{T}{d\log\bigl(\frac{dK}{\delta}\bigr)}
\Longrightarrow
2^{L_T}
\lesssim\sqrt{\frac{T}{d\log\bigl(\frac{dK}{\delta}\bigr)}}.
\]
\end{proof}

\subsection{Proof of Theorem~\ref{theorem; regret bound without gap}}
\begin{proof}
Under the event \(\Ecr_T\), by combining Lemma~\labelcref{lemma; key 3,lemma; regret formula without gap} and Lemma~\ref{lemma; upper bound phase}, we get
\begin{align*}
&\Regret(T)\\
&\lesssim d 2^{L_T}(\log (\frac{dKL_T}{\delta })+ L_T) +d^{\frac{3}{2}}\log(\frac{dK L_T}{\delta})L_T \\
&\lesssim d \sqrt{\frac{T}{d\log(\frac{dK}{\delta})}}\left(\log (\frac{dKL_T}{\delta })+  \log(\frac{T}{d\log(\frac{dK}{\delta})})\right) +d^{\frac{3}{2}}\log(\frac{dK L_T}{\delta}) L_T \; \; \text{(by Lemma~\ref{lemma; upper bound phase})}\\
&\stackrel{}{\lesssim} \sqrt{\frac{dT}{\log(\frac{dK}{\delta})}} (\log(\frac{dK}{\delta}) + \log L_T + \log(\frac{T}{d}) ) +d^{\frac{3}{2}}\log(\frac{dKT}{\delta}) L_T\\
&\lesssim \sqrt{\frac{dT}{\log(\frac{dK}{\delta})}} (\log(\frac{K}{\delta}) + \log T + \log \log T)  +d^{\frac{3}{2}}\log(\frac{dKT}{\delta}) \log(T/d)\\
&\lesssim \sqrt{dT \log(\frac{K}{\delta}) } + \sqrt{dT} \log T +d^{\frac{3}{2}}\log(\frac{dKT}{\delta})\log(T/d).
\end{align*}
Since \(\PP\bigl[\Ecr_T\bigr] \geq 1-\delta\) by Lemma~\ref{lemma; key 1}, the result follows.
\end{proof}

\subsection{Proof of Theorem~\ref{theorem; regret bound with gap}}

\begin{proof}
Under the event \(\Ecr_T\), by Lemma~\ref{lemma; key 2}, the optimal arm \(a^\star\) is contained in every phase \(\ell\).
Consider the phase \(\ell\) with \(\ell > L_\star\).
Under the event \(\Ecr_T\), by Lemma~\ref{lemma; key 3}, the optimal arm \(a^\star\) is the unique arm in the phase \(\ell\), i.e. $\cA_\ell = \{ a^\star\}$.
Recall that 
\begin{align*}
    2^{L_\star} \asymp \frac{1}{\Delta_\star} \asymp \frac{1}{\varepsilon_{L_\star}}.
\end{align*}
Using the Lemma~\ref{lemma; key 3}, under the event \(\Ecr_T\), the regret is bounded by
\begin{align*}
\Regret(T) 
&\lesssim \sum_{\ell=1}^{L_\star} n_{\ell}\varepsilon_{\ell-1}
\lesssim \sum_{\ell=1}^{L_\star} n_{\ell}\varepsilon_{\ell}\\
&\lesssim \sum_{\ell=1}^{L_\star}\Bigl(\frac{d\log\bigl(\frac{dK\ell}{\delta \varepsilon_\ell}\bigr)}{\varepsilon_\ell^2} 
+\frac{d^{\frac{3}{2}}}{\varepsilon_\ell}\log\bigl(\frac{dK \ell }{\delta}\bigr)\Bigr)\varepsilon_\ell \\
&\lesssim \sum_{\ell=1}^{L_\star} \frac{d\log\bigl(\frac{dKL_\star}{\delta \varepsilon_\ell}\bigr)}{\varepsilon_\ell} 
+ d^{\frac{3}{2}}\log\bigl(\frac{dK \ell}{\delta}\bigr)\\
&\stackrel{}{\lesssim} d2^{L_\star}\log\Bigl(\frac{dK}{\delta \varepsilon_{L_\star}}\Bigr) + d2^{L_\star}\log L_\star +
d^{\frac{3}{2}} \log\Bigl(\frac{dK L_\star}{\delta}\Bigr) L_\star\\
&\stackrel{}{\lesssim} d2^{L_\star}\log\Bigl(\frac{dK}{\delta\Delta_\star}\Bigr) + d2^{L_\star}\log \log(\frac{1}{\Delta_\star}) +
d^{\frac{3}{2}} \log\Bigl(\frac{dK L_\star}{\delta}\Bigr) L_\star\\
&\lesssim d2^{L_\star}\log\Bigl(\frac{dK}{\delta \Delta_\star}\Bigr) 
+ d^{\frac{3}{2}}\log\Bigl(\frac{dK}{\delta \Delta_\star}\Bigr)L_\star.
\end{align*}
Hence, finally we get
\[
\Regret(T) 
\lesssim \Bigl(\frac{d}{\Delta_\star} + d^{\frac{3}{2}}\Bigr)\log \Bigl(\frac{dK}{\delta\Delta_\star}\Bigr) \log(\frac{1}{\Delta_\star}).
\]
Since \(\PP\bigl[\Ecr_T\bigr] \geq 1-\delta\) by Lemma~\ref{lemma; key 1}, we get the desired result.
\end{proof}

\subsection{Proof of Theorem~\ref{Theorem; Exp of PE}}

\paragraph{Proof of BAI.}
As shown in the proof of Theorem~\ref{theorem; regret bound with gap}, under the event \(\Ecr_T\), the arm \(a_\star\) is the unique remaining arm after phase \(\ell > L_\star\).
Since $\sum_{\ell=1}^{L_\star} n_\ell \asymp n_{L_\star} =  \tilde \cO(\frac{d}{\Delta_\star^2})$, we get the wanted result.

\paragraph{Proof of PAC.}
Under the event \(\Ecr_T\), by Lemma~\ref{lemma; key 3}, in phase \(\ell\), each remaining arm has suboptimality gap smaller than \(4\varepsilon_\ell = \frac{1}{2^{\ell-1}}\). 
Then we obtain the result directly.

\hfill \BlackBox

\section{\(K\)-Independent Results}\label{section; K-independent results}
\noindent
Next, we present a result without \(K\) (action cardinality) dependence. 
When \(\log K \gtrsim d\), i.e., \(K \gg d\), the previously established results may be suboptimal.
First, using Lemma~\ref{lemma; self normalized bound second moment}, we can easily get an estimation error bound and a sample complexity without explicit \(K\)-dependence.

\begin{corollary}[\(K\)-independent estimation error bound]\label{corollary: K-independent error bound}
When we run pure exploration with the policy \(\pb^{\des}\), the estimator at time \(t\) with ridge regularizer $\beta_t = \log(t/\delta)$, named \(\hat{\theta}_t\), satisfies 
\[
\max_{i \in [K]} \bigl\lvert (x_i - x_1)^\top\bigl(\hat{\theta}_t - \theta^\star\bigr)\bigr\rvert 
\lesssim  \frac{d\sqrt{\log\bigl(\frac{t}{\delta}\bigr)}}{\sqrt{t}}.
\]
\end{corollary}

\begin{corollary}[\(K\)-independent PAC bound]\label{corollary; K-independent sample complexity}
For a fixed \(\varepsilon>0\), using the policy \(\pb^{\des}\), we sample from that policy for \(t\) rounds and then exit, as same as Section~\ref{subsection; pure exploration}. 
Then if 
\[
t \geq  C_3 \frac{d^2}{\varepsilon^2} \log\Bigl(\frac{d}{\varepsilon \delta}\Bigr)
=\tilde{\Ocal}\Bigl(\frac{d}{\varepsilon^2}\Bigr),
\]
for some absolute constant \(C_3\), our greedy policy with the estimator \(\hat{\theta}_t\) is \((\varepsilon, \delta)\)-PAC.
\end{corollary}
Corollary~\ref{corollary: K-independent error bound} and Corollary~\ref{corollary; K-independent sample complexity} give the sample complexity to achieve an \(\varepsilon\)-estimation error of the estimator.
This coincides with the known optimal result for linear bandits, without any dependence on \(K\).

\paragraph{Adaptive Low-Regret Algorithm.}
For \(\algob\), there is a sampling part with each arm pulled \(n_\ell\) times.
With a slight modification of the sampling number \(n_\ell\) in our low-regret algorithm \(\algob\), we can obtain an adaptive regret that enjoys 
\[
\min\Bigl(\tilde{\Ocal}\bigl(\sqrt{dT\log K}\bigr), \tilde{\Ocal}\bigl(d\sqrt{T}\bigr)\Bigr)
\]
while maintaining PAC and BAI properties.
This algorithm is exactly the same as \(\algob\), except setting 
\[
n_\ell 
= 4 \min\Bigl( 
C_2\bigl\lceil  \frac{d}{\varepsilon_\ell^2}\log\bigl(\frac{dK\ell(\ell+1)}{\delta\varepsilon_\ell}\bigr) +\frac{d^{\frac{3}{2}}}{\varepsilon_\ell}\log\bigl(\frac{dK \ell (\ell+1)}{\delta}\bigr) \bigr\rceil  ,
C_3\bigl\lceil \frac{d^2}{\varepsilon_\ell^2}\log\bigl(\frac{d\ell(\ell+1)}{\delta\varepsilon_\ell}\bigr)\bigr\rceil
\Bigr).
\]
We then have the following regret bound for this modified algorithm.

\begin{theorem}[Adaptive regret bound]\label{theorem; adaptive regret bound}
The adaptive version of \algob\ has cumulative regret bound
\[
\Regret(T) \lesssim \min\Bigl(\tilde{\Ocal}\bigl(\sqrt{dT\log K}\bigr),\tilde{\Ocal}\bigl(d\sqrt{T}\bigr)\Bigr)
\]
with probability at least \(1-\frac{1}{T^2}\).
\end{theorem}

\paragraph{Discussion.}
Our regret bound is robust for the case \(d \ll K\).
It is strictly sharper than the previous works \citet{kim2019contextual,krishnamurthy2018semiparametric}, which exhibit \(\tilde \cO(d\sqrt{T})\) or \(\tilde \cO(d^{\frac{3}{2}} \sqrt{T})\)-type bounds.

\section{Proofs for $K$-independent Results}\label{section; proof K independent results}

\subsection{Cauchy-Schwarz Based Error Analysis}\label{appendix; self-normalized concentrations}
\noindent
In this section, we present the concentration of the estimation error bound using the self-normalized inequality from \citet{krishnamurthy2018semiparametric}, specifically Lemma 11 of that work.
We consider the situation in which arms are sampled from a fixed policy.

\begin{lemma}[Modified Lemma 11 from \citealt{krishnamurthy2018semiparametric}]\label{lemma; modified cmu self normalization concentration}
Under Assumption~\ref{assumption; boundedness} (boundedness), with probability at least $1-\delta$, the following holds for $t$:
\begin{align*}
\|\hat{\theta}_t-\theta^\star \|_{\widehat{\Vb}_t +\beta \Ib_d} \leq \sqrt{\beta}+\sqrt{27 d \log \bigl(1+\frac{t}{d\delta}\bigr)+54 \log \bigl(\frac{t}{\delta}\bigr)},
\end{align*}
for \(\beta = \log (t/\delta)\) under the fixed policy.
\end{lemma}

The original selection of \(\beta\) is \(\beta \asymp d\log (t/\delta)\); however, using Lemma~\ref{lemma; trace class concentration bounded}, we can relax it for the fixed policy case.
Even though we choose the larger \(\beta \asymp d\log (t/\delta)\), it yields the same bound when we use that lemma.

\begin{proof}
Following the proof of Lemma 11 of \citet{krishnamurthy2018semiparametric}, the only role of the ridge regularizer \(\beta\) is to guarantee
\[
\frac{1}{c}(\widehat{\Vb}_t + \beta \Ib_d) \preceq \Vb_t + \beta \Ib_d  \preceq c(\widehat{\Vb}_t + \beta \Ib_d ).
\]
Using Lemma~\ref{lemma; trace class concentration bounded}, we can prove that \(\beta \asymp \log (t/\delta)\) is enough to satisfy this relation.
The remaining proof is exactly the same as in \citet{krishnamurthy2018semiparametric}.
\end{proof}

\begin{lemma}[Cauchy-Schwarz based estimation error bound]\label{lemma; self normalized bound second moment}
Under the boundedness assumption, with probability at least $1-\delta$, the following holds for $t$:
\begin{align*}
\sqrt{t}\|\hat{\theta}_t-\theta^\star\|_{\bSigma + \lambda \Ib_d} = \|\hat{\theta}_t-\theta^\star\|_{{\Vb}_t +\beta \Ib_d} \lesssim \sqrt{d\log(t/\delta)}.
\end{align*}
when we choose \(\beta=  \log (t/\delta)\) under the fixed policy.
Here, $c$ is an absolute constant.   
\end{lemma}

\begin{proof}
Using Lemma~\ref{lemma; trace class concentration bounded}, our choice of \(\beta\) makes the following with probability at least $1-\frac{\delta}{10}$:
\[
\frac{1}{c}(\widehat{\Vb}_t + \beta \Ib_d) \preceq  \Vb_t + \beta \Ib_d \preceq c(\widehat{\Vb}_t + \beta \Ib_d),
\]
and it follows directly by combining Lemma~\labelcref{lemma; modified cmu self normalization concentration}.
\end{proof}

\subsection{Proof of Corollary~\ref{corollary: K-independent error bound}}
\noindent
For any $z \in \Xcal-x_1$, using Lemma~\ref{lemma; self normalized bound second moment}, we get 
\begin{align*}
|z^\top (\hat{\theta}-\theta^\star)| &\leq    \|z \|_{(\bSigma_{\des} + \lambda \Ib_d)^{-1}}\|\hat{\theta}_t-\theta^\star\|_{\bSigma_{\des} + \lambda \Ib_d} \\
&\lesssim  \|z \|_{(\bSigma_{\des} + \lambda \Ib_d)^{-1}} \frac{\sqrt{d \log(t/\delta)}}{\sqrt{t}} \\
&\lesssim \frac{d\sqrt{ \log(t/\delta)}}{\sqrt{t}}
\end{align*}
with probability $1-\delta$.

\subsection{Proofs of Corollary~\ref{corollary; K-independent sample complexity}.}
\noindent
We can use the same argument from the proof of Corollary~\ref{corollary; dimension optimal sample complexity}, combined with the result of Corollary~\ref{corollary: K-independent error bound}. 

\hfill \BlackBox

\subsection{Proof of Theorem~\ref{theorem; adaptive regret bound}}
\noindent
We now prove our theorem.
\begin{lemma}[Good events]
With probability at least $1-2\delta$,
\begin{align*}
|x^T(\hat{\theta}_{(\ell)}-\theta^\star)| \leq \frac{\varepsilon_{\ell}}{2}
\end{align*}
for any arm $x \in \cA -\cA_\ell(1)$ and any phase $1 \leq \ell \leq L_T$.
We call this event \(\Fcr_T\).
\end{lemma}

\begin{proof}
We use the same argument from the proof of Lemma~\ref{lemma; key 1}.
Combining Corollary~\ref{corollary; dimension optimal sample complexity} and Corollary~\ref{corollary; K-independent sample complexity}, we can prove it in the same way.
\end{proof}

\begin{lemma}\label{lemma; key 4}
Under the event \(\Fcr_T\), for each phase \(\ell\), the action set \(A_\ell\) contains \(a^\star\).
    Any arm $a$ in $\cA_\ell$ satisfies 
    \begin{align*}
     x_a^\top \theta^\star >    x_{a^\star}^\top \theta^\star -4\varepsilon_{\ell}.
    \end{align*}
\end{lemma}

\begin{proof}
    The same as the proof of Lemma~\labelcref{lemma; key 2,lemma; key 3}.
\end{proof}

\paragraph{Main Proof of Theorem~\ref{theorem; adaptive regret bound}}
Define
\begin{align*}
n_\ell &= \min\left(
4 \bigl\lceil   C_4 \frac{d^2}{\varepsilon_\ell^2} \log\bigl(\frac{ d\ell (\ell+1)}{ \varepsilon \delta} \bigr)\bigr\rceil  , 
\bigr\lceil C_2\frac{d\log\bigl(\frac{dK \ell(\ell+1)}{\delta \varepsilon_\ell }\bigr)}{\varepsilon_\ell^2} 
+ C_3\frac{d^{\frac{3}{2}} \log^2\bigl(\frac{dK \ell(\ell+1)}{\delta}\bigr)}{\varepsilon_\ell}\bigr\rceil \right)  \\
&:= \min(n_\ell^{(1)}, n_\ell^{(2)}).
\end{align*} 
For each time \(t\), define the phase \(\ell(t)\) that \(t\) belongs to and define \(e(t):=\frac{1}{2^{\ell(t)}}\).
Under the event $\Fcr_T$, by Lemma~\ref{lemma; key 4}, we have
\begin{align*}
\Regret(T) &\lesssim \sum_{t=1}^T e(t) \\
&\lesssim \sum_{t=1}^T \min(e_1(t),e_2(t)) \\
&\lesssim \min\Bigl(\sum_{t=1}^T e_1(t),\sum_{t=1}^T e_2(t)\Bigr).
\end{align*}
First, using the same argument from the proof of Theorem~\ref{theorem; regret bound without gap}, we have 
\begin{align*}
\sum_{t=1}^T e_1(t)  \lesssim \sqrt{dT \log\bigl(\frac{K}{\delta}\bigr) } + \sqrt{dT} \log T +d^{\frac{3}{2}}\log\bigl(\frac{dKT}{\delta}\bigr).
\end{align*}
Next, we aim to bound \(\sum_{t=1}^T e_2(t)\).
Let \(L_T^{(k)}\) for \(k \in \{0,1\}\) be the smallest value of \(L\) satisfying
\begin{align*}
\sum_{\ell=1}^{L} n_{\ell}^{(k)} \geq T.
\end{align*}
Also, observe that
\begin{align*}
\sum_{t=1}^T e_2(t) &\lesssim \sum_{\ell=1}^{L_T^{(2)}} n_{\ell}^{(2)} \varepsilon_{\ell-1} \\
&\lesssim \sum_{\ell=1}^{L_T^{(2)}} n_{\ell}^{(2)} \varepsilon_{\ell}\\
&\lesssim \sum_{\ell=1}^{L_T^{(2)}} \frac{d^2\log \bigl(\frac{d \ell(\ell+1)}{\varepsilon \delta}\bigr)}{\varepsilon_\ell^2}  \varepsilon_\ell \\
& \lesssim d^2  \log \bigl(\frac{d L_T^{(2)}(L_T^{(2)}+1)}{\varepsilon_{L_T^{(2)}} \delta}\bigr)\sum_{\ell=1}^{L_T^{(2)}}\frac{1}{\varepsilon_\ell}  \\
&\lesssim d^2  \left(\log \bigl(\frac{d L_T^{(2)}}{\delta}\bigr) + L_T^{(2)} \right)L_T^{(2)}2^{L_T^{(2)}}.
\end{align*}
Since 
\begin{align*}
T &\gtrsim \sum_{\ell=1}^{L_T^{(2)}-1}  d^2 \log(1/\delta) 4^\ell \\
&\gtrsim  d^2 \log(1/\delta) 4^{L_T^{(2)}},
\end{align*}
we get the upper bound of \(L_T^{(2)}\) as 
\begin{align*}
2^{L_T^{(2)}} \lesssim \sqrt{\frac{T}{ d^2 \log(1/\delta)}}.
\end{align*}
Using the above result, we finally get
\begin{align*}
\sum_{t=1}^T e_2(t) \lesssim  d \sqrt{T} \log\Bigl(\frac{T}{d}\Bigr)\log\Bigl(\frac{dT}{\delta}\Bigr) = \tilde{\cO}(d\sqrt{T}).
\end{align*}
By setting $\delta = \frac{1}{T^2}$, it completes the proof.

\section{Suboptimality of Cauchy-Schwarz Based Analysis}\label{section; suboptimality of Cauchy}
\noindent 
Using the self-normalized bound of \citet{krishnamurthy2018semiparametric,kim2019contextual}, one obtains the following bound with probability at least \(1-\delta\) for all \(t\):
\begin{align*}
\|\hat{\theta}_t-\theta^\star \|_{\widehat{\Vb}_t +\beta \Ib} \leq \sqrt{\beta}+\sqrt{27 d \log \bigl(1+1/d \delta\bigr)+54 \log \bigl(\frac{t}{\delta}\bigr)}
\end{align*}
for \(\beta \asymp d\log \bigl(\frac{t}{\delta}\bigr)\).
All studies \citet{krishnamurthy2018semiparametric,kim2019contextual,choi2023semi} used this method.
It gives 
\begin{align*}
\|\hat{\theta}_t-\theta^\star \|_{\widehat{\Vb}_t +\beta \Ib_d} \lesssim \sqrt{d \log \bigl(\frac{t}{\delta}\bigr)}.
\end{align*}
Using this, they bound the estimation error as
\begin{align*}
|z^\top(\widehat{\theta}_t -\theta^\star )| \lesssim \|z \|_{(\widehat \bSigma_t + \lambda_t \Ib_d)^{-1}} \frac{1}{\sqrt{t}} \times \sqrt{d\log\bigl(\frac{t}{\delta}\bigr)}.
\end{align*} 
If $z \notin \operatorname{span}(\cX-x_1)$, the lower bound result of Theorem~\ref{theorem; dimension optimal error bound} implies we may not get finite estimation error bound.
For \(z \in \operatorname{span}(\cX-x_1)\), for sufficiently large \(t\), we have 
\begin{align*}
\|z \|_{(\widehat{\bSigma}_t + \lambda_t \Ib_d)^{-1}} \asymp   2\|z \|_{\bSigma_\pb^{-1}} 
\end{align*} 
and get
\begin{align*}
|z^\top(\hat{\theta}_t -\theta^\star )| \lesssim \|z \|_{\bSigma_\pb^{-1}}\frac{ \sqrt{d\log(\frac{t}{\delta})}}{\sqrt{t}},
\end{align*}
which has the additional factor \(\sqrt{d}\) compared to Theorem~\ref{theorem; dimension optimal error bound}.

\section{Discussion for Computational Efficiency and Size of Absolute Constants}\label{subsection; computational efficiency}
\noindent

\paragraph{Compuational Effeiciency of \algob}
BOSE algorithm \citep{krishnamurthy2018semiparametric} requires to solve minimax quadratic optimization problem every time and it requires heavy \(\Omega(T)\) computations.
Semiparametric-TS algorithm \citep{kim2019contextual} requires to calculate \textbf{exact} probability of Thompson sampling policy and it requires large sampling every time.
Both makes heavy computations (quadratic optimization, gram matrix inverse, sampling) with $\Omega(T)$ times.
Our algorithm requires \(\Ocal(\log T)\) times of policy updates, so we need to calculate poicy and matrix inverse at most \(\Ocal(\log T)\) times.
Also, if there is a gap, then \(\Ocal(\log(1/\Delta_{\min}))\) policy update is required, which is significantly computational efficient.

\paragraph{Size of Absolute Constants.}
The absolute constants in our analysis depend critically on the constants from Theorem~\ref{theorem; optimal variance design} and Theorem~\ref{theorem; dimension optimal error bound}.
The constant in Theorem~\ref{theorem; optimal variance design} is specified as value $2$.
The constant in Theorem~\ref{theorem; dimension optimal error bound}, as detailed in its proof, is a product of the constant for second-moment comparability (from \cref{equation; constant for comparability of second moment}) and constants from the standard Bernstein inequality. The proof is presented in Appendix~\ref{section; proof error analysis}. The constants from standard concentration inequalities are well-known and are typically mild.

We now investigate the magnitude of the constant $c$ related to second-moment comparability, which satisfies:
\begin{align}\label{equation; constant for comparability of second moment}
    \frac{1}{c}(\widehat \Vb_t + \log (t/\delta) \Ib) \preceq \Vb_t + \log(t/\delta) \Ib \preceq c(\widehat \Vb_t + \log (t/\delta) \Ib)
\end{align}
The following lemma shows that this constant $c$ is also mild. In our algorithm, the length of each phase is on the order of $n_\ell \gtrsim \frac{d}{\varepsilon^2} +\frac{d^{3/2}}{\varepsilon}$, so the condition in Lemma~\ref{lemma:concentration} is easily satisfied.

\begin{lemma}\label{lemma:concentration}
Under the same setup as Theorem~\ref{theorem; dimension optimal error bound}, when $t \gtrsim M\log(d/\delta)$, the following holds with a probability of at least \(1-\delta/10\):
    \begin{align*}
        \frac{1}{2}\left(\widehat{\bSigma}_t + \frac{\log (t/\delta)}{t} \Ib\right) \preceq \bSigma + \frac{\log (t/\delta)}{t} \Ib \preceq  \frac{3}{2} \left(\widehat{\bSigma}_t  + \frac{\log (t/\delta)}{t} \Ib\right)
    \end{align*}
Recall that when we use our design \algoa, we have $M =d$.
\end{lemma}

\begin{proof}
In the proof of Theorem~\ref{theorem; dimension optimal error bound}, for $\lambda_t = \log(t/\delta)/t$, we showed that with a probability of at least \(1-\delta/10\):
\begin{align*}
\left\|(\bSigma + \lambda_t  \Ib_d)^{-\frac{1}{2}}\left(\bSigma - \widehat{\bSigma}_t\right)(\bSigma + \lambda_t  \Ib_d)^{-\frac{1}{2}}\right\|_{\op}
\lesssim \sqrt{\frac{M\log\left(\frac{10d}{\delta}\right)}{t}}
\lesssim \sqrt{\frac{M\log\left(\frac{d}{\delta}\right)}{t}}.
\end{align*}
Hence, when \(t \gtrsim d\log(d/\delta)\), we have:
\begin{align*}
\left\|(\bSigma + \lambda_t  \Ib_d)^{-\frac{1}{2}}\left(\bSigma - \widehat{\bSigma}_t\right)(\bSigma + \lambda_t  \Ib_d)^{-\frac{1}{2}}\right\|_{\op}
\leq \frac{1}{2},
\end{align*}
which leads to the desired result.
\end{proof}

\section{Experiments}\label{appendix; experiments}
\noindent
We conducted experiments with two shifts \(\nu_t^1 =1 + \sin(2t)\) and \(\nu_t^2 = \max(\frac{\log (t+1)}{5},2) \times (-1)^{q_t}\) where \(q_t \) is remainder of \(t\) dividing by \(3\).
We first simulate the cumulative regret of \algob \  and compare them with two existing semiparametric bandit algorithms: BOSE from \citet{krishnamurthy2018semiparametric} and semiparametric-TS of \citet{kim2019contextual}.
Second, to demonstrate our algorithm is PAC, we examine the maximum estimation error, $\mathbf{e}_t :=\max_{i \in [K]} |(x_i-x_1)^\top (\hat{\theta}_t-\theta^\star) |$ over time \(t\) and we plot the value of \(\sqrt{t} \cdot \mathbf{e}_t\) over time. 
If \(\sqrt{t} \cdot \mathbf{e}_t\) is bounded by some constant, it means that our algorithm has \(\sqrt{t}\)-rate PAC property.

\subsection{Cumulative Regret}
\noindent
For the comparison of cumulative regret, we conducted to cases: \(d=5, K=10\) and \(d=20, K =30\) with noise level \(\sigma = 1\).
We constructed feature set with suboptimality gap \(\Delta_\star = 1/2\).
All previous literature \citet{kim2019contextual,krishnamurthy2018semiparametric} set one parameter as tuning parameter, so we also regard the exploration paramaters as hyperparameters.

\paragraph{Hyperparameters.}
For each algorithm, we tuned a single exploration parameter to its theoretical value, up to a constant factor.

For semiparametric-TS algorithm \citep{kim2019contextual}, we set $v$, which is an exploration parameter for Thompson sampling, as \(v= c\sqrt{d\log T} \) for \(c = 1,2,4,8\) and we chose the optimal hyperparameter.

For BOSE algorithm, we set ridge regularizer $\lambda = 4 d \log (9 T)+8 \log (4 T / \delta)$ as suggested in their algorithm and we only tune $\gamma$, which is a parameter related to exploration level of the algorithm.
We set \(\gamma = c (\sqrt{\lambda}+\sqrt{d \log(\frac{T}{d})+\log ( T / \delta)})\) for \(c=1,2,4,8\), up to constant with their theoretical value.

For our algorithm, we set \(n_\ell = c \min \bigg(\big \lceil\frac{d\log(\frac{d K \ell }{\delta \varepsilon_\ell})}{\varepsilon_\ell^2} + \frac{d^{\frac{3}{2}} \log(\frac{dK \ell }{\delta})}{\varepsilon_\ell}\big\rceil,  \lceil  \frac{d^2}{\varepsilon_\ell^2} \log(\frac{d\ell}{\varepsilon_\ell \delta} )\rceil \bigg)\) for hyperparameter \(c \in \{1,2,4,8\}\) and used the best one.

We can see that our algorithm \algob 's plot became flat after certain time. 
It means that our algorithm did a best arm identification.
Although the graph shapes for the sine shift and cosine shift appear similar in the case of \(d=20, K=30\), the actual regret values differ: the cumulative regret of our algorithm for \(\nu_t^1\) is 2677.66 but for \(\nu_t^2\) is 2703.10.

\begin{figure}[H]
  \centering
  \begin{subfigure}[]{}
  \includegraphics[width=0.43\textwidth]{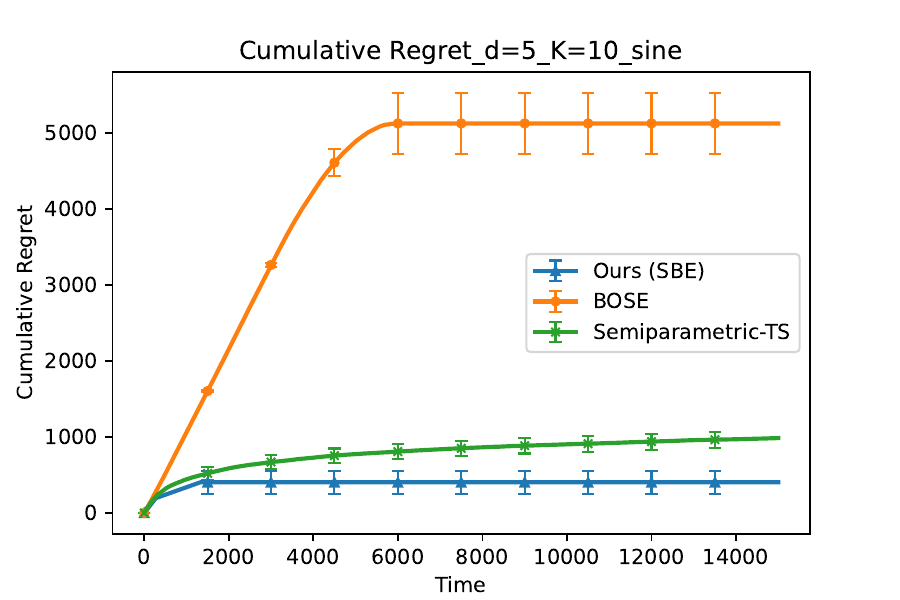}
  \label{fig:f2}
  \end{subfigure}
  \begin{subfigure}[]{}
  \includegraphics[width=0.43\textwidth]{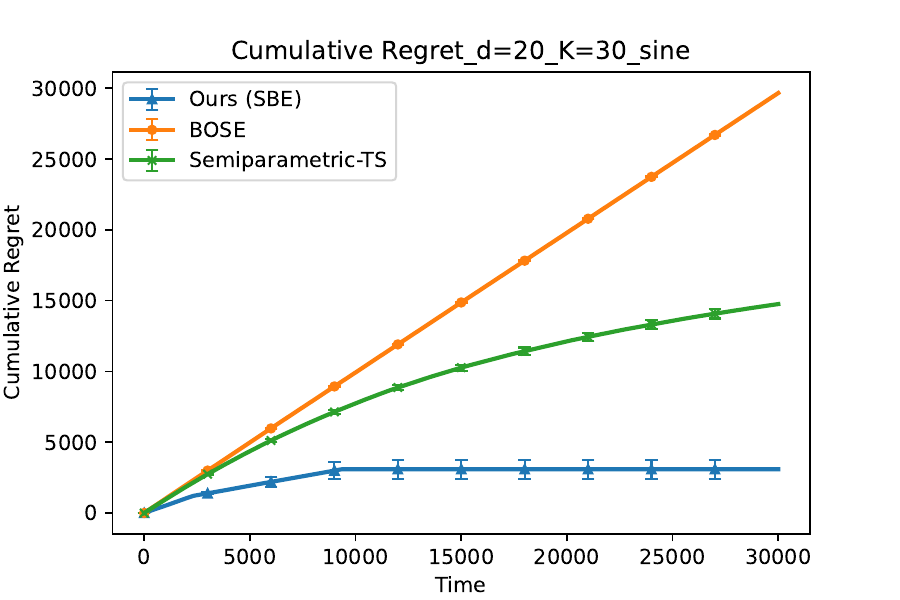}
  \label{fig:f2}
  \end{subfigure}
  \caption{Plots (a) and (b) are the cumulative regret plots of \algob \  with sine shift \(\nu_t^1\). 
  We conducted two cases: \(d=5, K=10\) and \(d= 20, K=30\).}
  \end{figure}
\begin{figure}[H]
\centering
\begin{subfigure}[]{}
\includegraphics[width=0.43\textwidth]{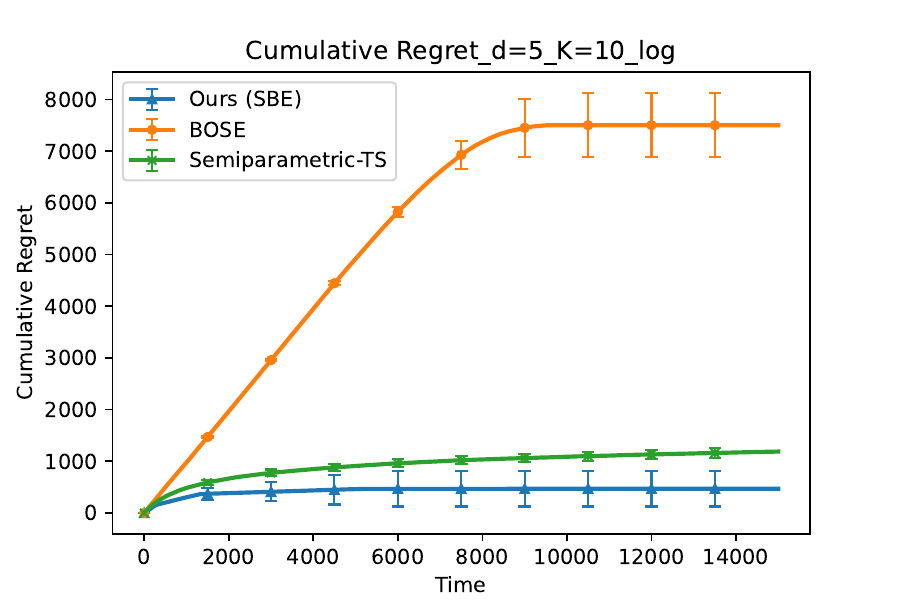}
\label{fig:f2}
\end{subfigure}
\begin{subfigure}[]{}
\includegraphics[width=0.43\textwidth]{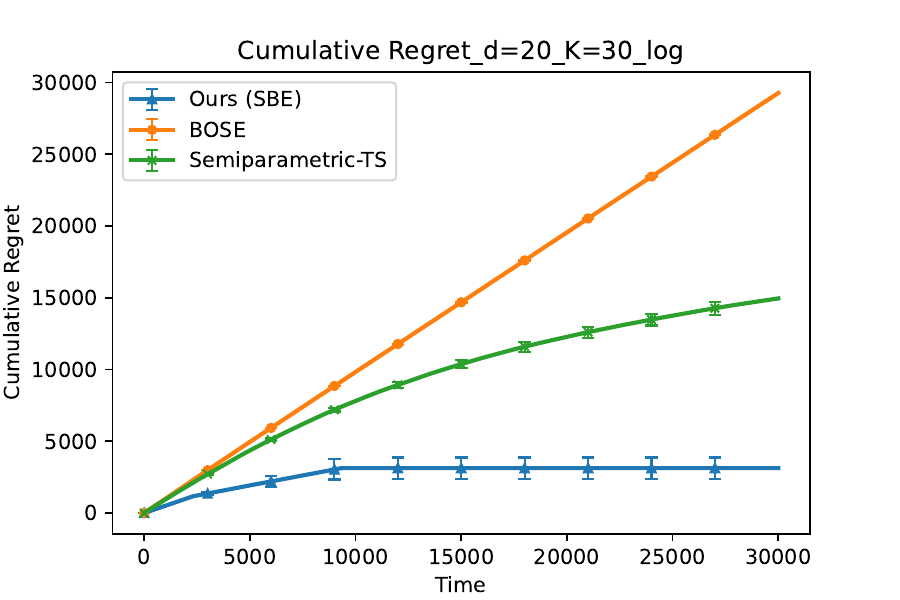}
\label{fig:f2}
\end{subfigure}
\caption{Plots (a) and (b) are the cumulative regret plots of \algob \  with log shift, \(\nu_t^2\).
 We conducted two cases: \(d=5, K=10\) and \(d= 20, K=30\).}
\end{figure}

\subsection{Experiments for PAC Property}
\noindent
Next we conducted experiment to investigate the maximum estimation error over arms, \[ \mathbf{e}_t := \max_{i \in [K]} |(x_i-x_1)^\top (\hat{\theta}_t- \theta^\star) |\] for fixed policy of \algoa.
We plot the $\sqrt{t} \mathbf{e}_t$ over time \(t \ge 1\).
We run 10 times and plot the error bar. 
We set same $K =30$, but change $d=5$ and $d=30$ to see the dimension dependency. 
The theoretical value is \(\sqrt{t}\mathbf{e}_t \lesssim \tilde{\cO}(\sqrt{d \log K})\) and we can see the value of \(\sqrt{t}\mathbf{e}_t \leq 10\sqrt{d \log K}\) for both cases.
Since we calculate every arm \(x_i, i \in[K]\)'s estimation error, this means that our design \algoa \  performs good pure exploration and matches our theory.

\begin{figure}[H]
\centering
\begin{subfigure}[]{}
\includegraphics[width=0.43\textwidth]{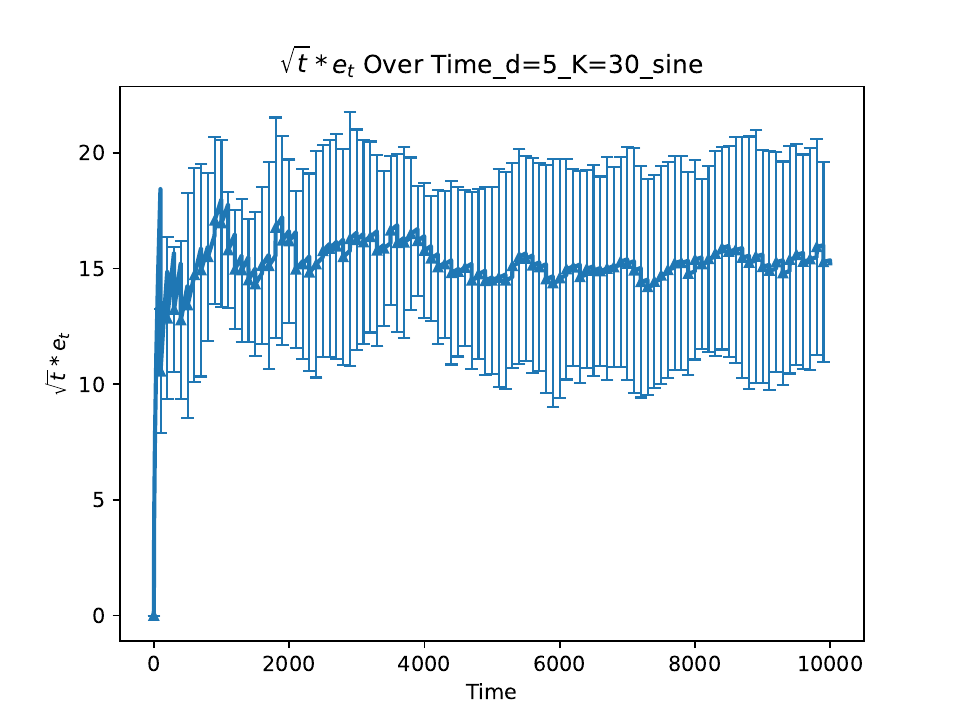}
\label{fig:f2}
\end{subfigure}
\begin{subfigure}[]{}
\includegraphics[width=0.43\textwidth]{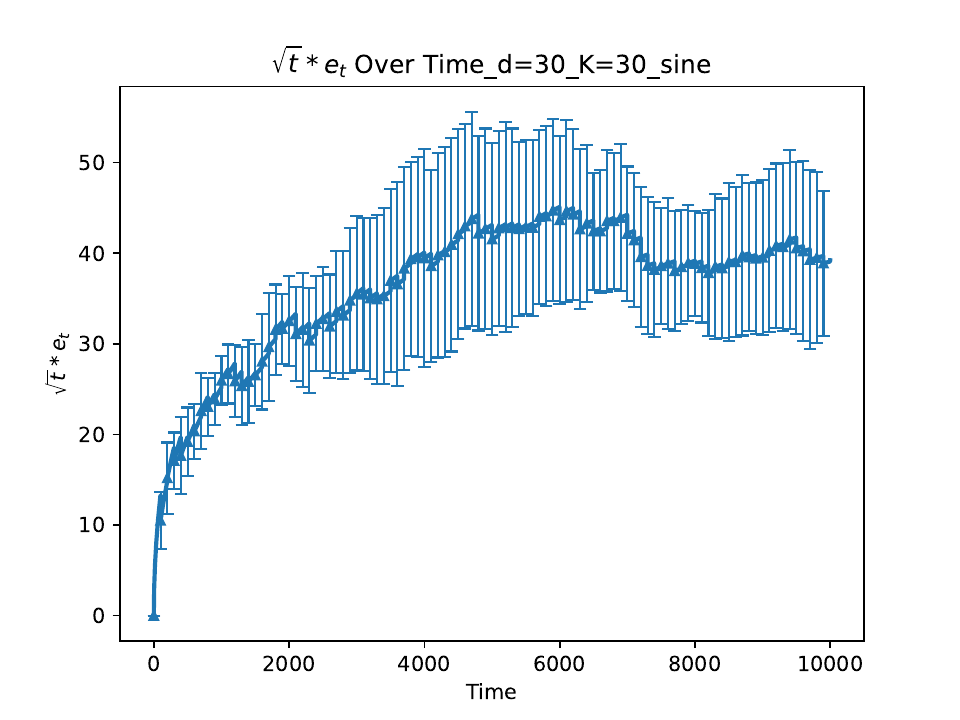}
\label{fig:f2}
\end{subfigure}
\caption{
(a), (b) are plots of \(\sqrt{t}\eb_t\) for \(d=5, K=30\) and \(d=30,K=30\).}
\end{figure}

\begin{figure}[H]
\centering
\begin{subfigure}[]{}
\includegraphics[width=0.43\textwidth]{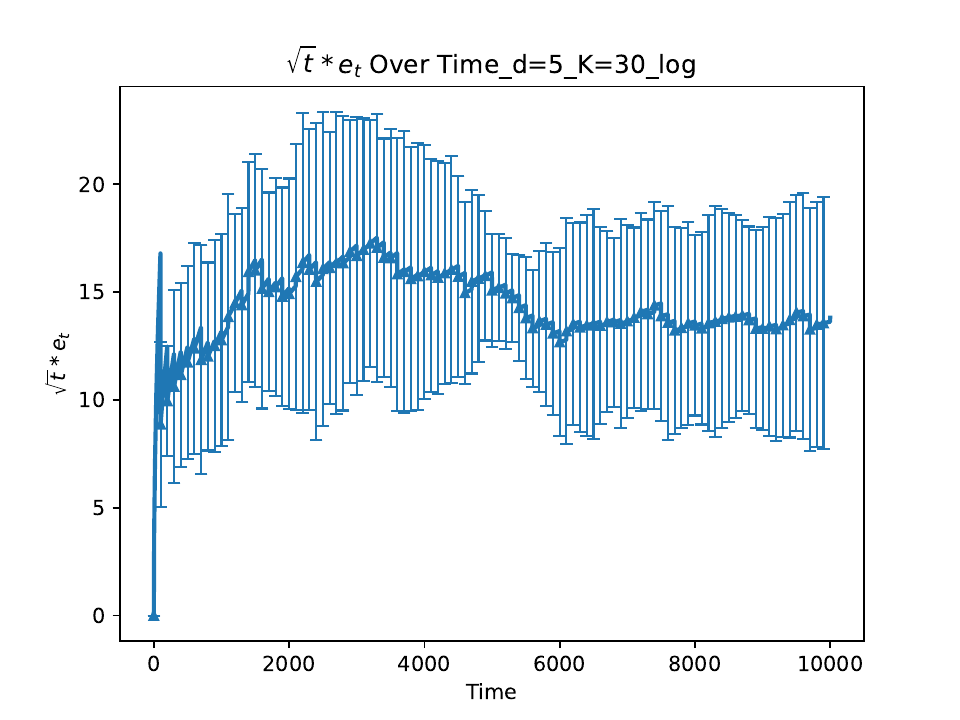}
\label{fig:f2}
\end{subfigure}
\begin{subfigure}[]{}
\includegraphics[width=0.43\textwidth]{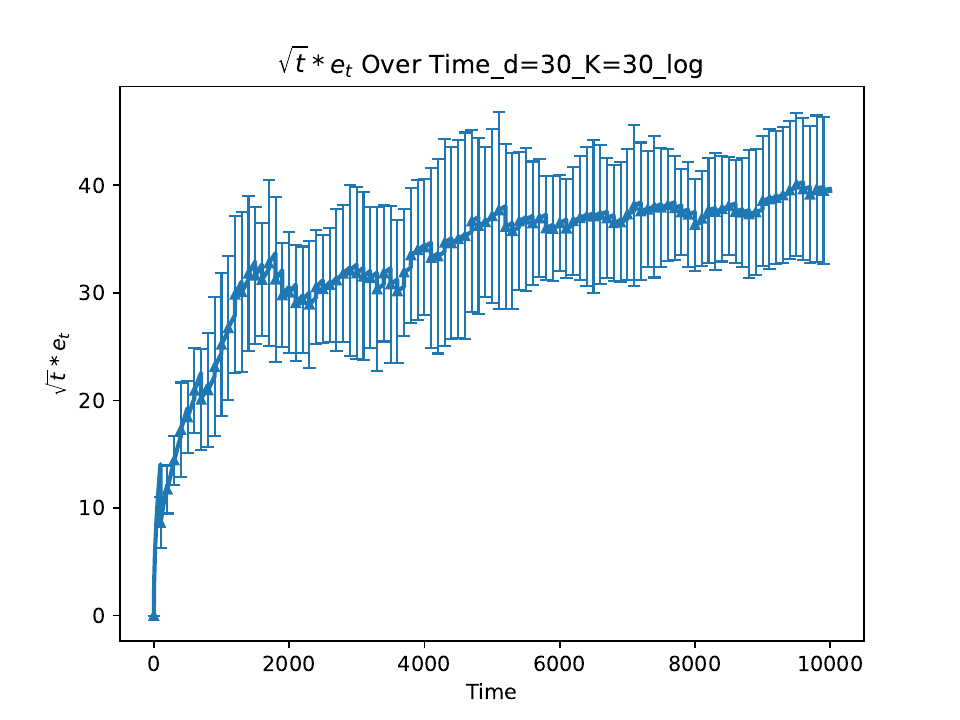}
\label{fig:f2}
\end{subfigure}
\caption{
(a), (b) is plot of \(\sqrt{t}\eb_t\) for \(d=5, K=30\) and \(d=30,K=30\).}
\end{figure}

\section{Technical Lemmas}\label{appendix; second moment concentrations}

\begin{lemma}\label{lemma; second moment concentration}
Assume \(z_i \sim \cP\) i.i.d. for \(1 \leq i \leq t\) and define \(\widehat \Vb_t := \sum_{i=1}^t z_iz_i^\top\) and \(\Vb_t = t\EE[z_1z_1^\top]\).
If $\|z_i\|_2 \le R$ for some absolute constant $R>0$ almost surely, then for \(\beta= \log(t/\delta)\), we have
\begin{align}\label{equation; comparability}
\frac{1}{c}( \widehat{\Vb}_t + \beta \Ib_d) \preceq  \Vb_t + \beta \Ib_d   \preceq c(\widehat{\Vb}_t + \beta \Ib_d)
\end{align}
holds with probability at least \(1-\frac{\delta}{10}\) for some absolute constant \(c>0\).   
\end{lemma}
\begin{proof}
By applying Lemma~\ref{lemma; trace class concentration bounded} directly, for some absolute constant $c_0 >1 $ and \(\xi =c_0 \log(\frac{t}{\delta})\), with probability at least \(1-\frac{\delta}{10}\), 
\begin{align*}
\frac{1}{2}( \widehat{\Vb}_t + \xi \Ib_d) \preceq  \Vb_t + \xi \Ib_d   \preceq 2(\widehat{\Vb}_t + \xi \Ib_d)
\end{align*}
holds.
Thus for \(\beta =  \log(\frac{t}{\delta})\), we have 
\begin{align*}
\widehat{\Vb}_t + \beta \Ib_d \succeq ( \widehat{\Vb}_t + \xi \Ib_d) \frac{\beta}{\xi} \succeq \frac{1}{c_0}( \widehat{\Vb}_t + \xi \Ib_d) \succeq \frac{1}{2c_0}( {\Vb}_t + \xi \Ib_d) \succeq \frac{1}{2c_0}( {\Vb}_t + \beta \Ib_d).
\end{align*}
Similarly, we have 
\begin{align*}
\widehat{\Vb}_t + \beta \Ib_d  \preceq 2c_0 ({\Vb}_t + \xi \Ib_d).
\end{align*}
Hence there exists absolute constant \(c>0\) with probability at least \(1-\frac{\delta}{10}\), 
\begin{align*}
\frac{1}{c}( {\Vb}_t + \beta \Ib_d) \preceq   \widehat{\Vb}_t + \beta \Ib_d  \preceq c( {\Vb}_t + \beta \Ib_d).
\end{align*}
\end{proof}

\begin{remark}\label{remark; regularizer choice}
We can choose \(\beta = c\log(t/\delta)\) for any absolute constant \(c\), but for simplicity, we just choose \(\beta = \log(\frac{t}{\delta})\).
\end{remark}

\begin{lemma}[Corollary E.1 from \citealt{wang2023pseudo}]\label{lemma; trace class concentration bounded}
Let $\left\{\boldsymbol{x}_i\right\}_{i=1}^n$ be i.i.d. random elements in a separable Hilbert space $\mathbb{H}$ with $\boldsymbol{\Sigma}=$ $\mathbb{E}\left(\boldsymbol{x}_i \otimes \boldsymbol{x}_i\right)$ being trace class. Define $\widehat{\boldsymbol{\Sigma}}=\frac{1}{n} \sum_{i=1}^n \boldsymbol{x}_i \otimes \boldsymbol{x}_i$. Choose any constant $\gamma \in(0,1)$ and define an event $\mathcal{A}=\{(1-\gamma)(\boldsymbol{\Sigma}+\lambda \boldsymbol{I}) \preceq \widehat{\boldsymbol{\Sigma}}+\lambda \boldsymbol{I} \preceq(1+\gamma)(\boldsymbol{\Sigma}+\lambda \boldsymbol{I})\}$.
1. If $\left\|\boldsymbol{x}_i\right\|_{\mathbb{H}} \leq M$ holds almost surely for some constant $M$, then there exists a constant $C \geq 1$ determined by $\gamma$ such that $\mathbb{P}(\mathcal{A}) \geq 1-\delta$ holds so long as $\delta \in(0,1 / 14]$ and $\lambda \geq \frac{C M^2 \log (n / \delta)}{n}$.    
\end{lemma}

\begin{lemma}[Lemma E.3 from \citealt{wang2023pseudo}]\label{lemma; operator covariance concentration}
Let $\left\{\boldsymbol{x}_i\right\}_{i=1}^n$ be i.i.d. random elements in a separable Hilbert space $\mathbb{H}$ with $\boldsymbol{\Sigma}=$ $\mathbb{E}\left(\boldsymbol{x}_i \otimes \boldsymbol{x}_i\right)$ being trace class. Define $\widehat{\boldsymbol{\Sigma}}=\frac{1}{n} \sum_{i=1}^n \boldsymbol{x}_i \otimes \boldsymbol{x}_i$.
1. If $\left\|\boldsymbol{x}_i\right\|_{\mathbb{H}} \leq S$ holds almost surely for some constant $S$, then for any $v^2 \geq\|\boldsymbol{\Sigma}\|$ and $r \geq \operatorname{Tr}(\boldsymbol{\Sigma}) / v^2$,

\begin{align*}
\mathbb{P}\left(\|\widehat{\boldsymbol{\Sigma}}-\boldsymbol{\Sigma}\|_{\op} \leq \sqrt{\frac{8 S^2 v^2 \log (r / \delta)}{n}}+\frac{6 S^2 \log (r / \delta)}{n}\right) \geq 1-\delta, \quad \forall \delta \in(0, r / 14]
\end{align*}

\end{lemma}

\end{document}